\documentclass[11pt]{article}

\usepackage{fullpage}
\usepackage[numbers]{natbib}
\usepackage{algorithm}
\usepackage[noend]{algorithmic}
\usepackage{amsmath,amsthm,amsfonts,amssymb}
\usepackage{amsmath}
\usepackage{hyperref}
\usepackage{color}
\usepackage{mathrsfs}
\usepackage{enumitem}
\usepackage{bm}
\usepackage{multirow}
\usepackage{booktabs}
\usepackage{makecell}
\usepackage{graphicx}
\usepackage{subfigure}
\usepackage{caption}
\usepackage{thmtools}
\usepackage{thm-restate}
\usepackage{setspace}
\usepackage{makecell}
\definecolor{light-gray}{gray}{0.85}
\usepackage{colortbl}
\usepackage{xcolor}
\usepackage{hhline}

\newtheorem{theorem}{Theorem}
\newtheorem{lemma}[theorem]{Lemma}
\newtheorem{corollary}[theorem]{Corollary}
\newtheorem{remark}[theorem]{Remark}

\newtheorem{proposition}[theorem]{Proposition}
\theoremstyle{definition}
\newtheorem{definition}[theorem]{Definition}

\newtheorem{assumption}{Assumption}

\hypersetup{
    colorlinks,
    linkcolor={blue!50!black},
    citecolor={blue!50!black},
}
\colorlet{linkequation}{blue}

\usepackage{natbib}
\usepackage{algorithm}
\usepackage[noend]{algorithmic}
\usepackage{amsthm}
\usepackage{amsmath,amsfonts,amssymb}
\usepackage{hyperref}
\usepackage{color}
\usepackage{mathrsfs}
\usepackage{enumitem}
\usepackage{bm}
\usepackage{multirow}
\usepackage{booktabs}
\usepackage{makecell}
\usepackage{graphicx}
\usepackage{caption}
\usepackage{thmtools}
\usepackage{thm-restate}
\usepackage{setspace}
\usepackage{makecell}
\definecolor{light-gray}{gray}{0.85}
\usepackage{colortbl}
\usepackage{xcolor}
\usepackage{hhline}

\hypersetup{
    colorlinks,
    linkcolor={blue!50!black},
    citecolor={blue!50!black},
}
\colorlet{linkequation}{blue}

  {\list{}{\leftmargin=0.3in\rightmargin=0.3in}\item[]}%
  {\endlist}

\newcommand{\argmin}{\mathop{\mathrm{argmin}}}
\newcommand{\argmax}{\mathop{\mathrm{argmax}}}

\renewcommand{\det}{\mathrm{det}}

\newcommand{\trans}{^{\top}}

\newcommand{\proj}{\mathcal{P}}

\newcommand{\E}{\mathbb{E}}
\renewcommand{\Pr}{\mathbb{P}}
\newcommand{\Var}{\text{Var}}

\newcommand{\up}[1]{\overline{#1}}
\newcommand{\low}[1]{\underline{#1}}

\newcommand{\paren}[1]{{\left( #1 \right)}}

\newcommand{\Tcal}{\mathcal{T}}
\newcommand{\Xcal}{\mathcal{X}}

\newcommand{\Vcal}{\mathcal{V}}

\newcommand{\cC}{\mathcal{C}}

\newcommand{\fmu}{\mu}
\newcommand{\fres}{\nu}

\newcommand{\N}{\mathbb{N}}
\newcommand{\R}{\mathbb{R}}

\newcommand{\D}{\mathbb{D}}

\newcommand{\cS}{\mathcal{S}}
\newcommand{\cA}{\mathcal{A}}

\newcommand{\cF}{\mathcal{F}}

\newcommand{\cH}{\mathcal{H}}
\newcommand{\cT}{\mathcal{T}}

\newcommand{\Fcal}{\mathcal{F}}
\newcommand{\Dcal}{\mathcal{D}}
\newcommand{\Acal}{\mathcal{A}}
\newcommand{\Scal}{\mathcal{S}}
\newcommand{\Lcal}{\mathcal{L}}

\newcommand{\Gcal}{\mathcal{G}}
\newcommand{\Zcal}{\mathcal{Z}}
\newcommand{\Ocal}{\mathcal{O}}
\newcommand{\Ecal}{\mathcal{E}}

\newcommand{\Bcal}{\mathcal{B}}

\newcommand{\reg}{ {\mathrm{Reg}}}

\newcommand{\Mcal}{\mathcal{M}}

\newcommand{\unif}{\mathrm{Uniform}}

\newcommand{\golf}{\textsc{Golf}}
\newcommand{\golfmg}{\textsc{Golf\_with\_Exploiter}}
\newcommand{\golfbr}{\textsc{Compute\_Exploiter}}

\newcommand{\samp}{\textsc{Sampling}}

\newcommand{\dedim}{\dim_\textrm{DE}}
\newcommand{\BEdim}{\dim_\textrm{BE}}

\newcommand{\dirac}{\Delta}

\newcommand{\Ffrak}{\mathfrak{F}}

\newcommand{\Ccal}{\mathcal{C}}

\newcommand{\mg}{{\rm MG}}

\newcommand{\cG}{\mathcal{G}}

\newcommand{\RealErr}{\varepsilon_{\rm real}}
\newcommand{\CompErr}{\varepsilon_{\rm comp}}
\newcommand{\pg}{\proj_\Gcal}
\newcommand{\pgh}{\proj_{\Gcal_h}}
\newcommand{\pf}{\proj_\Fcal}
\newcommand{\pfh}{\proj_{\Fcal_h}}

\begin{document}

\title{The Power of Exploiter: Provable Multi-Agent RL\\ in Large State Spaces}

\author{
 Chi Jin\thanks{Princeton University. Email: \texttt{chij@princeton.edu}}
 \and
  Qinghua Liu\thanks{Princeton University. Email: \texttt{qinghual@princeton.edu}}
  \and
  Tiancheng Yu\thanks{MIT. Email: \texttt{yutc@mit.edu}}
}

\date{June 6, 2021; Revised: October 6, 2021}

\maketitle

\begin{abstract}

Modern reinforcement learning (RL) commonly engages practical problems with large state spaces, where function approximation must be deployed to approximate either the value function or the policy. 
While recent progresses in RL theory address a rich set of RL problems with general function approximation, such successes are mostly restricted to the single-agent setting. 
It remains elusive how to extend these results to multi-agent RL, 
especially in the face of new game-theoretical challenges.
This paper considers two-player zero-sum Markov Games (MGs). We propose a new algorithm that can provably find the Nash equilibrium policy using a polynomial number of samples, for any MG with low \emph{multi-agent Bellman-Eluder dimension}---a new complexity measure adapted from its single-agent version \citep{jin2021bellman}. A key component of our new algorithm is the exploiter, which facilitates the learning of the main player by deliberately exploiting her weakness. Our theoretical framework is generic, which applies to a wide range of models including but not limited to tabular MGs, MGs with linear or kernel function approximation, and MGs with rich observations.





\end{abstract}

\tableofcontents


\section{Introduction}
\label{sec:intro}
Multi-agent reinforcement learning (MARL) systems have recently achieved significant success in many AI challenges including 
the game  Go \citep{silver2016mastering,silver2017mastering}, Poker \citep{brown2019superhuman}, real-time strategy games \citep{vinyals2019grandmaster, openaidota}, decentralized controls or multiagent robotics systems \citep{brambilla2013swarm}, autonomous driving \citep{shalev2016safe}, as well as complex social scenarios such as hide-and-seek~\citep{baker2020emergent}.
Two crucial components that contribute to these successes are \emph{function approximation} and \emph{self-play}. Function approximation is frequently used in modern applications with large state spaces, where either the value function or the policy is approximated by parametric function classes, which are typically deep neural networks. Meanwhile, self-play enables the learner to improve by playing against itself instead of traditional human experts.

Despite the empirical success of MARL, existing theoretical guarantees in MARL only apply to the basic settings where the value functions can be represented by either tables (in cases where the states and actions are discrete) \citep{bai2020provable,bai2020near,liu2020sharp} or linear maps \citep{xie2020learning,chen2021almost}. While a recent line of works \citep{jiang2017contextual,wang2020provably,yang2020bridging, jin2021bellman,du2021bilinear} significantly advance our understanding of RL with general function approximation, and provide sample-efficient guarantees for RL with kernels, neural networks, rich observations, and several special cases of partial observability, they are all restricted to the single-agent setting. Distinct from  single-agent RL, each agent in MARL are facing not only the unknown environment, but also the opponents that can constantly adapt their strategies in response to the behavior of the learning agent. This additional game-theoretical feature  makes it challenging to extend the single-agent general function approximation results to the multi-agent setting.
This motivates us to ask the following question:
\begin{center}
\textbf{Can we design sample-efficient MARL algorithms for general function approximation?}
\end{center}
By ``sample-efficient'', we mean the algorithms provably learn within a polynomial number of samples that is independent of the number of states.
This paper provides the first positive answer to this question in the context of two-player zero-sum Markov Games (MGs) \citep{shapley1953stochastic,littman1994markov}. In particular, we make the following contributions:
\begin{itemize}
	\item We design a new self-play algorithm for MARL---\golfmg. Our algorithm maintains a main player and an exploiter, where the exploiter facilitates the learning of the main player by deliberately exploiting her weakness. Our algorithm features optimism, and can simultaneously address multi-agent, general function approximation, as well as the trade-off between exploration and exploitation.
	\item We introduce a new general complexity measure for MARL---\emph{Multi-agent Bellman Eluder (BE) dimension}, which is adapted from its single-agent version \citep{jin2021bellman}. We prove that our algorithm can learn the Nash equilibrium policies of any MARL problem with low multi-agent BE dimension, using a number of samples that is polynomial in all relevant parameters, but independent of the size of the state space. We further provide an online guarantee for our algorithm when facing against adversarial opponents.
	\item We show that our theoretical framework applies to a wide range of models including  tabular MGs, MGs with linear or kernel function approximation, and MGs with rich observations. Formally, we prove that all examples above admit low multi-agent BE dimension. 
	Particularly, for linear MGs,  we obtain a $\tilde{\Ocal}(H^2d^2/\epsilon^2)$ sample complexity bound for finding $\epsilon$-approximate Nash policies, which improves upon the cubic dependence on dimension $d$ in  \cite{xie2020learning}. 
\end{itemize}

We remark that our algorithm is sample-efficient but \emph{not} computationally efficient. Designing computationally efficient algorithms in the context of general function approximation is an open problem even in the single-agent setting \citep{jiang2017contextual,jin2021bellman,du2021bilinear}, which we left as an interesting topic for future research.

\textbf{Technical Challenges} One of the most important questions in the design of RL algorithms is the choice of \emph{behavior policies}---the policies that are used to collect samples. For a given general function class $\cF$, the single-agent RL algorithms simply use behavior policy in the form of $\pi_f$, which is the greedy policy with respect to certain optimistic value function $f \in \cF$, i.e., $\pi_f(s) = \argmax_a f(s, a)$. The choice becomes more sophisticated in the two-player setting. All existing algorithms that can perform provably efficient exploration \citep{bai2020provable,bai2020near,liu2020sharp,xie2020learning,chen2021almost} choose the behavior policy as the general-sum Nash equilibrium or coarse correlated equilibrium of two value functions $f, g \in \cF$, which are the upper and lower bound of the true value functions respectively. It turns out that this approach is only effective under \emph{optimistic closure}---a strong assumption which is true for tabular and several special linear settings, but does not hold for many practical applications (see Appendix \ref{sec:technical} for more discussions). In contrast, our exploiter framework provides a new approach for the theory community to 
choose behavior policies---the main agent plays the optimistic Nash policy based on the existing knowledge while the exploiter focuses on exploiting the main agent. Our approach does not rely on strong assumptions such as optimistic closure.  We hope our exploiter framework could facilitate the future development of provable MARL algorithms.


\subsection{Related Works}
\label{sec:related}


There is an extensive literature on empirical  MARL, in distributed, cooperative, and competitive settings \cite[see, e.g.,][and the references therein]{lowe2017multi,rashid2018qmix,vinyals2019grandmaster,berner2019dota,li2020multi,yu2021surprising}. Due to space limit, we focus on reviewing the theoretical works in this section.


\paragraph{Markov Game.} Markov Game (MG), also known as stochastic game \citep{shapley1953stochastic}, is a popular model in multi-agent RL \citep{littman1994markov}. Early works have mainly focused on finding Nash equilibria of MGs with known transition and reward \citep{littman2001friend,hu2003nash,hansen2013strategy,wei2020linear}, or under strong reachability conditions such as simulators \citep{wei2017online,jia2019feature,sidford2020solving,zhang2020model,wei2021last} where exploration is not needed.

A recent line of works provide non-asymptotic guarantees for learning two-player zero-sum tabular MGs, in the setting that requires strategic exploration. \cite{bai2020provable} and \cite{xie2020learning} develop the first provably-efficient learning algorithms in MGs based on optimistic value iteration.  \cite{bai2020near} and  \cite{liu2020sharp} improve upon these works and achieve best-known sample complexity for model-free and model-based methods, respectively. Several extensions are also studied, including multi-player general-sum MGs \citep{liu2020sharp}, unknown games \footnote{This terminology stems from game theory, which means the actions of the opponents are not observed.} \citep{tian2020provably}, vector-valued MGs \citep{yu2021provably}, etc.

Beyond tabular MGs, \cite{xie2020learning} and \cite{chen2021almost} also study learning  MGs with linear function approximation. Their techniques heavily rely on ``optimistic closure'' (see Appendix \ref{sec:technical} for more details), which can not be directly extended to the general function approximation setting with weak assumptions. To our best knowledge, this paper provides the first positive result on learning MGs with general function approximation.


\paragraph{Markov Decision Process.} There has been a long line of research studying Markov Decision Process (MDP), which can be viewed as a single-agent version of Markov Game. Tabular MDP has been studied thoroughly in recent years \citep{brafman2002r,jaksch2010near,dann2015sample,azar2017minimax,zanette2019tighter,jin2018q,zhang2020almost}. Particularly, in the episodic setting, the minimax regret or sample complexity is achieved by both model-based \citep{azar2017minimax} and model-free \citep{zhang2020almost} methods, up to logarithmic factors. When the state space is large, the tabular formulation of RL is impractical and function approximation is necessary. The most studied case is linear function approximation \citep{jin2020provably,wang2019optimism,cai2019provably,zanette2020learning,zanette2020provably,agarwal2020flambe,neu2020unifying}.

For general function approximation, there are two common measures of complexity: Eluder dimension \citep{osband2014model,wang2020provably} and Bellman rank \citep{jiang2017contextual}, which are unified by a more generic notion---Bellman-Eluder dimension \citep{jin2021bellman}. Recently, \cite{du2021bilinear} develops a new problem class termed bilinear class, which can be considered as an infinite-dimensional extension of low Bellman rank class. Low complexity RL problems under the above criteria are statistically tractable. Polynomial sample complexity guarantees \citep{jiang2017contextual,sun2019model,jin2021bellman,foster2020instance},  or even regret guarantees \citep{wang2019optimism,dong2020root,jin2021bellman,wang2020provably} are established, under additional completeness and realizability conditions. However, most of the above algorithms are not computationally efficient and it 
remains open how to design computationally efficient algorithms for these settings.

\paragraph{Extensive-form games}
Finally, we remark that there is another long line of research on MARL based on the model of extensive-form games (EFG) \cite[see, e.g.,][]{koller1992complexity,gilpin2006finding,zinkevich2007regret,brown2018superhuman,brown2019superhuman,celli2020no}. Results on learning EFGs do not directly imply results for learning MGs, since EFGs are naturally tree-structured games, which can not efficiently represent MGs---graph-structured games where a state at the $h^{\text{th}}$ step can be the child of multiple states at the $(h-1)^{\text{th}}$ step.

\section{Preliminaries}

Markov Games (MGs) are the generalization of standard  Markov Decision Processes into the multi-player setting, where each player aims to maximize her own utility. 
In this paper, we consider \emph{two-player zero-sum episodic} MG, formally denoted by $\mg(H,\Scal,\Acal,\Bcal,\Pr,r)$, where $H$ is the number of steps per episode, $\Scal$ is the state space, $\Acal$ and $\Bcal$ are the action space for the max-player and the min-player respectively, $\Pr=\{\Pr_h\}_{h\in[H]}$  is the collection of transition measures with $\Pr_h(\cdot\mid s,a,b)$ defining the distribution of the next state after taking action $(a,b)$ at state $s$ and step $h$, $r=\{r_h\}_{h\in[H]}$ is the collection of reward functions with $r_h(s,a,b)$ equal to the reward   received after  taking action $(a,b)$ at state $s$ and step $h$.\footnote{In this paper, we study deterministic reward functions for cleaner presentation. All our results  immediately generalize to the setting with stochastic reward.}
 Throughout this work, we assume the reward is nonnegative, and $\sum_{h=1}^H r_h(s_h,a_h,b_h)\le 1$  for all possible sequence $(s_1,a_1,b_1,\ldots,s_H,a_H,b_H)$.

 Each episode starts from a \emph{fixed initial state} $s_1$. At each step $h\in[H]$, two players observe $s_h\in\Scal$ and choose their  own actions $a_h\in\Acal$ and $b_h\in\Bcal$ simultaneously. Then, both players observe the actions of their opponents and receive reward $r_h(s_h,a_h,b_h)$. After that, the environment will transit to $s_{h+1}\sim \Pr_h(\cdot\mid s_h,a_h,b_h)$. Without loss of generality, we assume the environment always transit to a terminal state $s_{\rm end}$ at step $H+1$ which marks the end of this episode.

\paragraph{Policy, value function} A Markov policy $\mu$ of the max-player is a collection of vector-valued functions $\{\mu_h:\ \Scal\rightarrow\Delta_\Acal\}_{h\in[H]}$, each mapping a state to a distribution in the probability simplex over $\Acal$. 
We use the $a^{\rm th}$ coordinate of $\mu_h(s)$  to refer to the probability of taking action $a$ at state $s$ and step $h$. 
Similarly, we can define a Markov policy $\nu$ over $\Bcal$ for the min-player.

We denote $V^{\mu,\nu}_h:\Scal\rightarrow\R$ as the  value function for policy $\mu$ and $\nu$ at step $h$, with $V^{\mu,\nu}_h(s)$ equal to the expected cumulative reward received under policy $(\mu,\nu)$, starting from state $s$ at step $h$ until the end of the episode:
\begin{equation*}
	V_h^{\mu,\nu}(s):= \E_{\mu\times\nu}\left[\sum_{t=h}^H r_{t}(s_{t},a_{t},b_{t}) \mid s_h = s \right].
\end{equation*}
Similarly, we denote $Q^{\mu,\nu}_h:\Scal\times\Acal\times \Bcal\rightarrow\R$ as the action-value function for policy $\mu$ and $\nu$ at step $h$:
\begin{equation*}
 	Q_h^{\mu,\nu}(s,a,b):= \E_{\mu\times\nu}\left[\sum_{t=h}^H r_{t}(s_{t},a_{t},b_{t}) \mid s_h = s,a_h=a,b_h=b \right].
\end{equation*}

With slight abuse of notation, we use $\Pr_h$ as an operator so that $[\Pr_h V](s,a,b):=\E_{s'\sim \Pr_h(\cdot\mid s,a,b)} V(s')$ for any value function $V$. We also define $[\D_\pi Q](s):=\E_{(a,b)\sim \pi(\cdot,\cdot\mid s)} Q(s,a,b)$ for any policy $\pi$ and action-value function $Q$. By the definition of value functions, we have the Bellman equations:
\begin{equation*}
Q_h^{\mu,\nu}(s,a,b) = [r_h+\Pr_h V_{h+1}^{\mu,\nu}](s,a,b), \quad V_h^{\mu,\nu}(s)=[\D_{\mu_h\times\nu_h} Q_h^{\mu,\nu}](s)
\end{equation*}
for all $(s,a,b,h)\in\Scal\times\Acal\times\Bcal\times[H]$.
And for step $H+1$, we have $V_{H+1}^{\mu,\nu}\equiv 0$.

\paragraph{Best response, Nash equilibrium} 
For any policy of the max-player $\mu$, there exists a \emph{best response} policy of the min-player $\nu^\dagger(\mu)$ satisfying 
$V_h^{\mu,\nu^\dagger(\mu)} = \inf_\nu V_h^{\mu,\nu}(s)$ for all $(s,h)$.
For cleaner notation, we denote $V_h^{\mu,\dagger} := V_h^{\mu,\nu^\dagger(\mu)}$. 
Similarly, we can define $\mu^\dagger(\nu)$ and $V_h^{\dagger,\nu}$. It is known \citep{filar2012competitive} that there exists policies $\mu^\star,\nu^\star$ that are optimal against the best responses of the oponents, i.e., 
$$
  V_h^{\mu^\star,\dagger}(s)  = \sup_\mu V_h^{\mu,\dagger}(s), \quad 
V_h^{\dagger,\nu^\star}(s)  = \inf_\nu V_h^{\dagger,\nu}(s), \quad {\rm for\ all\ } (s,h).
$$
We call these optimal policies $(\mu^\star,\nu^\star)$ a Nash equilibrium of the Markov game, which are further known to satisfy the following minimax equation:
$$
  \sup_\mu\inf_\nu V_h^{\mu,\nu}(s)
=V_h^{\mu^\star,\nu^\star}(s)
= \inf_\nu \sup_\mu V_h^{\mu,\nu}(s)
, \quad {\rm for\ all\ } (s,h).
$$
We call the value functions of $(\mu^\star,\nu^\star)$  Nash value functions, and abbreviate $V_h^{\mu^\star,\nu^\star}$  as $V_h^\star$ and $Q_h^{\mu^\star,\nu^\star}$ as $Q_h^\star$. 
Intuitively speaking, in a Nash equilibrium, no player can benefit by unilaterally deviating from her own policy.

\paragraph{Learning objective} We say a policy $\mu$ of the max-player is an \textbf{$\epsilon$-approximate Nash} policy if $V_1^{\mu,\dagger}(s_1) \ge V_1^\star(s_1) - \epsilon$. Suppose an agent interacts with the environment for $K$ episodes, and denote by $\mu^k$ the policy executed by the max-player in the $k^{\rm th}$ episode. Then the (cumulative) regret of the max-player is defined as
\begin{equation*}
 \reg(K): = \sum_{k=1}^K \left[V^{\star}_1(s_1) - V_1^{\mu^k,\dagger}(s_1)\right].
\end{equation*}
The goal of reinforcement learning is to learn $\epsilon$-approximate Nash policies or to achieve sublinear regret.
In this paper, we focus on learning the Nash policies of the max-player. By symmetry, 
the definitions and techniques directly extend to learning the Nash policies of the min-player.

\subsection{Function Approximation} 
This paper considers reinforcement learning with value function approximation. 
Formally, the learner is provided with a function class $\Fcal=\Fcal_1\times \dots \times \Fcal_H$ where each $\Fcal_h \subseteq (\Scal\times\Acal\times\Bcal\rightarrow [0,1])$ consists of candidate functions to approximate some target action-value functions at step $h$. Since there is no reward at step $H+1$, we set $f_{H+1}\equiv 0$ for any $f\in\Fcal$ without loss of generality.

\paragraph{Induced policy and value function}
For any value function $f\in\Fcal$, we denote by $\mu_f  = (\mu_{f, 1}, \ldots, \mu_{f, H})$ the Nash policy of the max-player induced by $f$, where 
\begin{equation*}
\mu_{f, h}(s) = \argmax_{\mu\in\Delta_\Acal} \min_{\nu\in\Delta_\Bcal}\mu\trans f_h(s,\cdot,\cdot) \nu \quad \mbox{for \ all \ } (s,h).
\end{equation*}
Moreover, we denote by $V_f$ the Nash value function induced by $f$,  so that
\begin{equation*}
	V_{f, h}(s) = \max_{\mu \in\Delta_\Acal} \min_{\nu\in\Delta_\Bcal}\mu\trans f_h(s,\cdot,\cdot) \nu\quad \mbox{for \ all \ } (s,h).
\end{equation*}
We further generalize the concept of induced value function by additionally introducing a superscript $\mu$ (which represents a Markov policy of the max-player) and define $V^\mu_f$ so that
\begin{equation*}
 	V_{f, h}^\mu(s) =  \min_{\nu\in\Delta_\Bcal}\mu_h(s)\trans f_h(s,\cdot,\cdot) \nu\quad \mbox{for \ all \ } (s,h).
\end{equation*}
Based on $V_f$ and $V_f^\mu$, two types of Bellman operators are defined as below
\begin{equation}
\left\{
	\begin{aligned}
		(\Tcal_{h} f)(s,a,b)  :=& r_h(s,a,b) + \E_{s'\sim\Pr_h(\cdot\mid s,a,b)}V_{f, h+1}(s')\\
	(\Tcal_{h}^{\mu} f)(s,a,b):=& r_h(s,a,b) + \E_{s'\sim\Pr_h(\cdot\mid s,a,b)}V^{\mu}_{f,h+1}(s')
	\end{aligned}\right.
\end{equation}
which naturally generalize the optimal Bellman operator and the policy Bellman operator from MDPs to MGs. In the remainder of this paper, we will refer to them as the Nash Bellman operator and $\mu$-Bellman operator, respectively.

It is known that RL with function approximation is in general  statistically intractable without further assumptions (see, e.g., hardness results in \cite{krishnamurthy2016pac,weisz2020exponential}). 
Below, we present two assumptions that are generalizations of commonly adopted assumptions  in MDP literature. 

\begin{assumption}[$\RealErr$-Realizability] \label{as:realizable}
    For any $f \in \cF$ and $h\in[H]$, there exist $Q_h, Q'_h \in \cF_h$ s.t.
    $\|Q_h^\star - Q_h\|_\infty \le \RealErr$ and $\|Q_h^{\mu_f, \dagger} -Q'_h\|_\infty  \le \RealErr$.
\end{assumption}

Realizability requires the function class to be well-specified so that the value function of the Nash policy and the value function of any induced policy $\mu_f$ against its best response lie inside $\Fcal$ approximately.

Let $\Gcal=\Gcal_1\times \dots \times \Gcal_H$ be an auxiliary function class provided to the learner  where each $\Gcal_h \subseteq (\Scal\times\Acal\times\Bcal\rightarrow [0,1])$.
Completeness requires the auxiliary function class $\Gcal$ to be rich enough so that applying Bellman operators to any function in the primary function class $\Fcal$ will end up in $\Gcal$ approximately.
\begin{assumption}[$\CompErr$-Completeness] \label{as:complete}
    For any $f,f' \in \cF$ and $h\in[H]$, there exist $Q_h, Q'_h \in \cG_h$ s.t.
    $\|\Tcal_{h} f_{h+1}- Q_h\|_\infty \le \CompErr$ and $\|\Tcal_{h}^{\mu_f}  f'-Q'_h\|_\infty  \le \CompErr$.
\end{assumption}

In the single-agent setting \citep{jin2021bellman}, the realizability and completeness conditions are stated for the special case $\RealErr = \CompErr = 0$, while this paper considers a strictly more general condition. This extension makes it possible to handle the misspecified setting, i.e.,  $\Fcal$ and $\Gcal$ only satisfy these conditions approximately, but not exactly.
Even when $\Fcal$ and $\Gcal$ satisfy these conditions exactly, the above extension can still help us avoid some technical redundancy caused by the possible infinite cardinality of $\Fcal$ and $\Gcal$, by only considering their finite coverings. To be precise,

\begin{definition}[$\epsilon$-Covering]
    The function class $\tilde{\cF} \subseteq \cF$ is an $\epsilon$-covering of $\cF$ if for any $f \in \cF$, there exist $\tilde{f} \in \tilde{\cF}$ s.t. $\|f_h-\tilde{f}_h\|_{\infty} \le \epsilon$, for any step $h$.
\end{definition}

The following lemma implies that we can always restrict our attention to finite coverings, which also  approximately satisfy the  realizability and completeness conditions up to satisfying accuracy. 

\begin{lemma}
	\label{lem:finite-cover}
    If $\cF$ and $\cG$ satisfy $\RealErr$-Realizability and $\CompErr$-Completeness, and $\tilde{\cF}$ and $\tilde{\cG}$ are $\epsilon$-coverings of $\cF$ and $\cG$ respectively, then $\tilde{\cF}$ and $\tilde{\cG}$ satisfy $(\RealErr+\epsilon)$-Realizability and $(\CompErr+\epsilon)$-Completeness.
\end{lemma}

As a result, we only need to consider finite function class $\cF$ and $\cG$ in Section \ref{sec:main-results}. Some of the examples introduced in Section~\ref{sec:example} involve infinite function classes. 
To address this inconsistency, we can first compute the finite coverings of $\cF$ and $\cG$, which satisfy the approximate realizability and completeness conditions by Lemma~\ref{lem:finite-cover}, and then invoke the algorithms in Section \ref{sec:main-results} with the coverings instead of the original function classes. 
It is straightforward to verify that by replacing the cardinality with the covering number, all theoretical guarantees derived in Section \ref{sec:main-results} still hold.

Finally, we define the $L^{\infty}$ Projection Operator, which will be frequently used in our technical statements and proofs.
\begin{definition}[$L^{\infty}$ Projection Operator]
For any function class $\cF$ and function $g$, $\pf(g) = \argmin_{f \in \cF}\|f-g\|_{\infty}$.
\end{definition}

\section{Main Results}
\label{sec:main-results}

In this section, we present an optimization-based algorithm \golfmg, and its theoretical guarantees for any multi-agent RL problem with low Bellman-Eluder dimension.

\subsection{Algorithm}

We describe \golfmg\ in Algorithm \ref{alg:golfmg}.
At a high level, \golfmg\ follows the principle of optimism in face of uncertainty, and maintains a global confidence set for Nash value function $Q^{\star}$ based on local constraints. 
It extends the single-agent 
\golf\ algorithm \citep{jin2021bellman} by introducing an exploiter subroutine---\golfbr, which facilitates the learning of the main player (the max player) by continuously exploiting her weakness.   

In each episode, \golfmg\ performs four main steps:
\begin{enumerate}
	\item Optimistic planning (Line~\ref{line:optimistic-planning}): compute the value function $f^k$ that induces the most optimistic Nash value from the current confidence set $\cC$, and choose $\mu^k$ to be its induced Nash policy of the max-player.  
	\item Finding exploiter (Line~\ref{line:compute-exploiter}): compute an approximate best response policy $\nu^k$ for the min-player, by invoking the subroutine \golfbr\ on historical data $\Dcal$ and the policy of the max-player $\mu^k$. 
	\item Data collection (Line~\ref{line:sample}-\ref{line:augment}): we provide two different options for the \samp\ procedure
	\begin{enumerate}[label=(\Roman*)]
		\item Roll a single trajectory by  following $(\mu^k,\nu^k)$.
		\item For each $h$, roll a trajectory by following $(\mu^k,\nu^k)$ in the first $h-1$ steps and take uniformly random action at step $h$. This option will roll $H$ trajectories each time, but in the $h^{\rm th}$ trajectory, only the $h^{\rm th}$ transition-reward tuple is augmented into $\Dcal_h$.
	\end{enumerate}
	\item Confidence set update (Line~\ref{line:confidence}): update confidence set $\cC$ using the renewed dataset $\Dcal$.
\end{enumerate}
 
The core components of \golfmg\ are the construction of the confidence set and the subroutine 
\golfbr, which we elaborate below in sequence.

\paragraph{Confidence set construction}
For each $h\in[H]$, \golfmg\ maintains a local constraint using the historical data $\Dcal_h$ at this step
\begin{equation} \label{eq:set_relaxed}
	\Lcal_{\Dcal_h}(f_h,f_{h+1}) \leq \inf_{g \in \Gcal_{h}} \Lcal_{\Dcal_h}(g,f_{h+1}) + \beta,
\end{equation}
where the squared loss $\Lcal_{\Dcal_h}$  is defined as 
\begin{equation}
	\label{equ:loss-1}
	\Lcal_{\Dcal_h}(\xi_h,\zeta_{h+1}) := \sum_{(s,a,b,r,s') \in \Dcal_h}\big[\xi_h(s,a,b)-r -V_{\zeta, h+1}(s')\big]^2.
\end{equation}
Intuitively speaking, $\Lcal_{\Dcal_h}$ is a proxy to the squared Bellman error under the Nash Bellman operator at step $h$. 
Unlike the classic Fitted Q-Iteration algorithm (FQI) \citep{szepesvari2010algorithms} where one simply updates $f_h \leftarrow \argmin_{\phi\in\Fcal_h}\Lcal_{\Dcal_h}(\phi,f_{h+1})$, our constraints implement a soft version of minimization so that any function with loss slightly larger than the optimal loss gets preserved. 
By making this relaxation, we can guarantee the true Nash value function $Q^\star$ contained in the confidence set $\Ccal$ with high probability.

\paragraph{Exploiter computation} The main challenge of learning MGs compared to learning MDPs lies in the choice of behavior policies as discussed in Section \ref{sec:intro}.
Motivated by the empirical breakthrough  work AlphaStar \citep{vinyals2019grandmaster}, we adapt the methodology of exploiter. 
Specifically, we design the \golfbr\ subroutine that approximately computes a best-response policy $\nu^k$ against the policy of the max-player $\mu^k$. 
By executing $\nu^k$, the min-player exposes the flaws of the max-player, and helps improve her  strategy. 
The pseudocode of \golfbr\ is given in Algorithm~\ref{alg:golfbr}, which basically follows the same rationale as Algorithm~\ref{alg:golfmg} except that we 
change the regression target by replacing $V_f$ with $V_f^\mu$, because the confidence set $\cC^\mu$ is constructed for the best-response value function $Q^{\mu,\dagger}$ instead of the Nash value function $Q^\star$. Formally, Algorithm~\ref{alg:golfbr} adapts a new square loss function defined below
\begin{equation}
	\label{equ:loss-2}
	 \Lcal_{\Dcal_h}^\mu(\xi_h,\zeta_{h+1}) := \sum_{(s,a,b,r,s') \in \Dcal_h}[\xi_h(s,a,b)-r -V^\mu_{\zeta, h+1}(s')]^2.
\end{equation}

Finally, we remark that the optimistic planning and exploiter computation steps are computationally inefficient in general. 
Designing computationally efficient algorithms in the context of general function approximation is an open problem even in the single-agent setting \citep{jiang2017contextual,jin2021bellman,du2021bilinear}, which we left as an interesting topic for future research.

%

%
%
%
%


\begin{algorithm}[t]
\caption{\golfmg\ $(\Fcal,\Gcal,K,\beta)$ }
\label{alg:golfmg}
\begin{algorithmic}[1]
 \STATE \textbf{Initialize}:  $\Dcal_1,\dots,\Dcal_H\leftarrow \emptyset$, $\Ccal \leftarrow \Fcal$.
 \FOR{\textbf{episode} $k$ from $1$ to $K$} 
 \STATE \textbf{Choose policy} $\mu^k = \mu_{f^k}$, where $f^k = \argmax_{f \in \Ccal}V_{f,1}(s_1)$. Let $\up{V}^k = V_{f^k,1}(s_1)$.
 \label{line:optimistic-planning}
 \STATE $(\nu^k,\low{V}^k) \leftarrow \golfbr(\Fcal,\Gcal,\beta,\Dcal,\mu^k)$. 
 \label{line:compute-exploiter}
 \IF{$\up{V}^k-\low{V}^k<\Delta$} \label{line:output-condition}
 \STATE \textbf{Output} $\mu^\text{out} = \mu^k$.
 \ENDIF
\STATE \textbf{Collect} samples \label{line:sample}\\
\quad \textbf{Option I}~~~execute policy $(\mu^k, \nu^k)$ and collect a trajectory \\
\qquad \qquad \qquad $(s_1, a_1, b_1, r_1, \ldots, s_H, a_H, b_H, r_H, s_{H+1})$.\\
\quad \textbf{Option II}~~~for all $h\in[H]$, execute policy $(\mu^k, \nu^k)$ at the first $h-1$ steps, take  uniformly \\
\qquad \qquad \qquad random action  at the $h^{\text{th}}$ step, and collect a sample $(s_h, a_h, b_h, r_h, s_{h+1})$.
\STATE \textbf{Augment} $\Dcal_h=\Dcal_h\cup \{(s_h, a_h, r_h, s_{h+1})\}$ for all $h\in[H]$.\label{line:augment}
\STATE \textbf{Update} \label{line:confidence}
\vspace{-2mm}
\begin{equation*}
	\Ccal=\left\{ f \in \Fcal:\  \Lcal_{\Dcal_h}(f_h,f_{h+1}) \leq \inf_{g \in \Gcal_{h}} \Lcal_{\Dcal_h}(g,f_{h+1}) + \beta \ \mbox{for all }h\in[H]\right\},	\vspace{-1mm}
\end{equation*}
\ENDFOR
\end{algorithmic}	
\end{algorithm}

\begin{algorithm}[t]
	\caption{\golfbr$(\Fcal,\Gcal,\beta,\Dcal,\mu)$}
	\label{alg:golfbr}
	\begin{algorithmic}[1]
	
	\STATE\label{line:br-1} \textbf{Construct}
	\vspace{-3mm}
	\begin{equation*}
		\Ccal^{\mu} =\left\{ f \in \Fcal:\  \Lcal^{\mu}_{\Dcal_h}(w_h,w_{h+1}) \leq \inf_{g \in \Gcal_{h}} \Lcal^{\mu}_{\Dcal_h}(g,w_{h+1}) + \beta \ \mbox{for all }h\in[H]\right\},
			\vspace{-3mm}
	\end{equation*}
	\STATE \label{line:br-2} \textbf{Compute} $\tilde{f} \leftarrow \argmin_{f \in \Ccal} V^{\mu}_{f, 1}(s_1)$ and $V = V^{\mu}_{\tilde{f}, 1}(s_1)$. 
	 \STATE \label{line:br-3} \textbf{Return} ($\nu^{\rm out}$,$V$) such that $\nu_h^{\rm out}(s)=\argmin_{\nu\in\Delta_\Bcal} \mu_h(s)\trans \tilde{f}_h(s,\cdot,\cdot)\nu$ for all $s,h$.
	 \end{algorithmic}	
\end{algorithm}

\subsection{Complexity Measure}

In this subsection, we introduce our new complexity measure---multi-agent Bellman-Eluder (BE) dimension, which is generalized from its single-agent version \citep{jin2021bellman}.
We will show in Section \ref{sec:example} that the family of low BE dimension problems includes many interesting multi-agent RL settings, such as tabular MGs, MGs with linear or kernel function approximation, and MGs with rich observation.

We start by recalling the definition of distributional Eluder (DE) dimension \citep{jin2021bellman}.
\begin{definition}[$\epsilon$-independence between distributions]
\label{def:ind_dist}
	Let $\Gcal$ be a function class defined on $\Xcal$, and	$\nu,\mu_1,\ldots,\mu_n$ be probability measures over $\Xcal$.
	We say 	$\nu$ is $\epsilon$-independent of $\{\mu_1,\mu_2,\ldots,\mu_n\}$ with respect to $\Gcal$ if there exists $g\in\Gcal$ such that  
	$\sqrt{\sum_{i=1}^{n} ( \E_{\mu_i} [g])^2}\le \epsilon$, but $|\E_{\nu}[g]| > \epsilon$. 
\end{definition}
\begin{definition}[Distributional Eluder (DE) dimension]
\label{def:DE}
Let $\Gcal$ be a function class defined on $\Xcal$, and $\Pi$ be a family of probability measures over $\Xcal$. 
	The  distributional Eluder dimension $\dedim(\Gcal,\Pi,\epsilon)$ is the length of the longest sequence $\{\rho_1, \ldots, \rho_n\} \subset \Pi$ such that there exists $\epsilon'\ge\epsilon$ where $\rho_i$ is $\epsilon'$-independent of $\{\rho_1, \ldots, \rho_{i-1}\}$ for all $i\in[n]$.
\end{definition}

DE dimension generalizes the original Eluder dimension \citep{russo2013eluder} from points to distributions. 
The main advantage of this generalization is that for RL problems with large state space, it is oftentimes easier to evaluate a function with respect to certain distribution family than to estimate it point-wisely. And for the purpose of this paper, we will focus on the following two specific distribution families. 
\begin{enumerate}
	\item $\Dcal_{\dirac}:=\{\Dcal_{\dirac,h}\}_{h\in[H]}$, where $\Dcal_{\dirac,h} = \{\delta_{(s, a,b)}(\cdot) ~|~ (s,a,b)\in\cS\times \cA\times\Bcal\}$, i.e., the collections of probability measures that put measure $1$ on a single state-action pair.
	\item $\Dcal_{\cF}:=\{\Dcal_{\cF,h}\}_{h\in[H]}$, where $\Dcal_{\cF,h}$ denotes the collections of all probability measures over 
	$\Scal\times\Acal\times\Bcal$ at the $h^\text{th}$ step, which can be generated by executing some $(\fmu_f,\fres_{f,g})$ with $f,g\in\Fcal$. Here 
	 $\fres_{f,g}$ is the best response to $\mu_f$ regarding to Q-value $g$:$$ \fres_{f,g,h}(s) : = \argmin_{\nu\in\Delta_\Bcal}\fmu_{f,h}(s)\trans g_h(s,\cdot,\cdot)\nu \quad \mbox{for all } (s,h).$$
\end{enumerate}

Now, we are ready to introduce our key complexity notion---multi-agent Bellman-Eluder (BE) dimension, which is simply the DE dimension on the function class of Bellman residuals, minimizing over the two aforementioned distribution families, and maximizing over all steps.
\begin{definition}[Bellman-Eluder dimension]
\label{def:bedim}
Let $\cH_\cF$ be the function classes of Bellman residuals where $\cH_{\cF, h} := \{f_h-\Tcal_{h}^{\mu_g} f_{h+1} ~|~ f, g\in\Fcal\}$. Then the Bellman-Eluder dimension is defined as 
\begin{equation}
\dim_{\textrm{BE}}(\Fcal, \epsilon)= \max_{h\in[H]}\min_{\Dcal\in\{\Dcal_\Delta,\Dcal_\cF\}} \dim_{\textrm{DE}}(\cH_{\cF,h},\Dcal_h,\epsilon).
\end{equation}
\end{definition}
We remark that the Bellman residual functions used in Definition \ref{def:bedim} take both state and action as input. By choosing an alternative class of Bellman residuals defined over the state space only, we can similarly define a variant of this notion---V-type BE dimension. For clean presentation, we defer the formal definition to  Appendix \ref{sec:type2}.

\subsection{Theoretical Guarantees}
\label{sec:guarantee-Q}
Now we are ready to present the theoretical guarantees.

\begin{theorem}[Regret of \golfmg]
\label{thm:golf-mg-regret}
Under Assumption \ref{as:realizable} and \ref{as:complete}, 
there exists an absolute constant $c$ such that for any $\delta\in(0,1]$, $K\in\N$, 
if we choose $\beta=c\cdot(\log(KH|\Fcal||\Gcal|/\delta)+K\CompErr^2+K\RealErr^2)$, then with probability at least $1-\delta$, for all $k\in[K]$, Algorithm \ref{alg:golfmg} with Option I satisfies:
$\reg(k)\le \mathcal{O}(H\sqrt{dk\beta})$, where $d=\dim_{\textrm{BE}}(\Fcal, 1/K)$ is the BE dimension.
\end{theorem}
Theorem \ref{thm:golf-mg-regret} claims that \golfmg\ can achieve $\sqrt{K}$ regret on any RL problem with low BE dimension, under the realizability and the completeness assumption. Moreover, the regret scales polynomially with respect to the length of episode, the BE dimension 
of function class $\Fcal$, and the log-cardinality of the two function classes. 
In particular, it is independent of the size of the state space, which is of vital importance for practical RL problems, where the state space is oftentimes prohibitively large.

By pigeonhole principle, we also derive a sample complexity guarantee for \golfmg.

\begin{corollary}[Sample complexity of \golfmg]
	\label{cor:golf-mg-pac}
Under Assumption \ref{as:realizable} and \ref{as:complete}, there exists absolute constants $c,c'$ such that for any $\epsilon>0$,  choose $\beta=c\cdot(\log(KH|\Fcal||\Gcal|/\delta)+K\CompErr^2+K\RealErr^2)$ and $\Delta = c'(H\sqrt{d\beta/K}+\epsilon)$, where $d=\dim_{\textrm{BE}}(\Fcal, \epsilon/H)$ is the BE dimension of $\Fcal$,then with probability at least $1-\delta$, the output condition (Line~\ref{line:output-condition}) will be satistied at least once in the first $K$ episodes. Furthermore, 
the output policy $\mu^{\text{out}}$ of Algorithm \ref{alg:golfmg} with Option I is $\Ocal(\epsilon+H\sqrt{d}(\RealErr+\CompErr))$-approximate Nash, if  $K \ge \Omega\big( (H^2d/\epsilon^2)\cdot\log(H|\Fcal||\Gcal|d/\epsilon)\big)$.
\end{corollary}

Corollary \ref{cor:golf-mg-pac} asserts that 
\golfmg\ can find an $\Ocal(\epsilon)$-approximate Nash policy for the max-player using at most 
$\tilde\Ocal(H^2\log(|\Fcal||\Gcal|)\dim_{\textrm{BE}}(\Fcal, \epsilon)/\epsilon^2)$ samples, which scales linearly in the BE dimension and the log-cardinality of the function classes.  Since our ultimate goal is to compute an approximate Nash policy, one can early terminate Algorithm~\ref{alg:golfmg} once the output condition (Line~\ref{line:output-condition}) is satisfied.
We remark that by symmetry \golfmg\ can  also compute an $\Ocal(\epsilon)$-approximate Nash policy for the min-player using the same amount of samples.

\begin{remark}[Generalization to function classes with infinite cardinality]
{Although  Theorem \ref{thm:golf-mg-regret}  and Corollary \ref{cor:golf-mg-pac} are derived for finite function classes, they can also apply to infinite function classes with bounded $\ell_\infty$-covering number by Lemma \ref{lem:finite-cover}. For example, they hold for linear function classes, with the log-cardinality  replaced by the log-covering number.}
\end{remark}

We also derive similar sample complexity guarantee in terms of V-type BE dimension for Algorithm \ref{alg:golfmg} with Option II, which can be found in Appendix \ref{sec:type2}. 

\subsection{Adversarial Opponents}
So far we have been focusing on the self-play setting, where the learner can control both players and the goal is to learn an approximate Nash policy. Another setting of importance is the online setting, where the learner can only control the max-player, and the goal is to achieve high cumulative reward against the adversarial min-player. 
Intuitively, it is reasonable to expect  \golfmg\ could still work in the online setting, because Theorem \ref{thm:golf-mg-regret} demonstrates that it can achieve sublinear regret competing against the exploiter, which is the strongest possible adversary. 
This is indeed the case. The only place we need to change in Algorithm~\ref{alg:golfmg} is 
Line \ref{line:compute-exploiter}: instead of computing $\nu^k$ using Algorithm \ref{alg:golfbr}, we let  the adversary pick policy $\nu^k$.
Before stating the theoretical guarantee, we introduce an online version of Bellman Eluder dimension.  

\begin{definition}[Online Bellman-Eluder dimension]
\label{def:obe}
Let $\cH_\cF$ be the function classes of Bellman residuals where $\cH_{\cF, h} := \{f_h-\Tcal_{h} f_{h+1} ~|~ f\in\Fcal\}$. Then the online Bellman-Eluder dimension is defined as:
\begin{equation}
\dim_{\textrm{OBE}}(\Fcal, \epsilon)= \max_{h\in[H]} \dim_{\textrm{DE}}(\cH_{\cF,h},\Dcal_{\Delta,h},\epsilon).
\end{equation}
\end{definition}
Compared to the BE dimension in Definition \ref{def:bedim}, this online version uses a smaller function class that includes only the residuals with respect to the Nash Bellman operator. Another difference is the choice of distribution family is now limited to only $\Dcal_\Delta$  because the learner cannot control the policy of the min-player in the online setting. Now, we are ready to present the theoretical guarantee.

\begin{theorem}[Regret against adversarial opponents]
\label{thm:golf-online}
Assuming $\|Q_h^\star-\pfh(Q_h^\star)\| \le \RealErr$ and $\|\Tcal_{h} f_{h+1} -\pgh(\Tcal_{h} f_{h+1})\| \le \CompErr$ for all $f \in \cF$ and $h\in[H]$, 
there exists an absolute constant $c$ such that for any $\delta\in(0,1]$, $K\in\N$, 
if we choose $\beta=c\cdot(\log(KH|\Fcal||\Gcal|/\delta)+K\CompErr^2+K\RealErr^2)$, then
with probability at least $1-\delta$, for all $k\in[K]$, Algorithm \ref{alg:golfmg} with Option I satisfies 
$\sum_{t=1}^k [V^{\star}_1(s_1) - V_1^{\mu^t, \nu^t}(s_1)] \le  \mathcal{O}(H\sqrt{dk\beta})$,
where $d=\dim_{\textrm{OBE}}(\Fcal, 1/K)$ is the online BE dimension.
\end{theorem}
Theorem \ref{thm:golf-online} claims that on any problem of low online BE dimension, \golfmg\ can achieve $\sqrt{K}$ regret  against an adversarial opponent. Moreover, the multiplicative factor scales linearly in the horizon length, and sublinearly in the online BE dimension  as well as the log-cardinality of the function classes. 
Comparing to the self-play regret guarantee, Theorem \ref{thm:golf-online} holds under a weaker condition, but only guarantees 
Algorithm \ref{alg:golfmg} can play favorably against the adversary, instead of the best-response policy. This is unavoidable because if the adversary is very weak, then it is impossible to learn Nash policies by playing against it.

\section{Examples}
\label{sec:example}

In this section, we introduce five concrete multi-agent RL problems with low BE dimension: tabular MGs, MGs with linear function approximation, MGs with kernel function approximation,  MGs with rich observation, and feature selection for kernel MGs, which generalize their single-agent versions \citep{zanette2020learning,yang2020bridging,krishnamurthy2016pac,agarwal2020flambe} from MDPs to MGs.
Except for tabular MGs, all above examples are new and can not be addressed by existing works \citep{xie2020learning,chen2021almost}.


\paragraph{Tabular MGs} Starting with the simplest scenario, we show that MGs with finite state-action space has BE dimension no larger than the size of its state-action space, up to a logarithmic factor.
\begin{proposition}[Tabular MGs $\subset$ Low BE dimension]\label{ex:tabular}
Consider  a tabular MG with state set $\Scal$ and action set $\Acal\times\Bcal$. 
Then any function class $\Fcal\subseteq (\Scal\times\Acal\times\Bcal\rightarrow [0,1])$ satisfies 
 $$\dim_{\textrm{BE}}(\Fcal,\epsilon) \le \Ocal\big( |\Scal\times\Acal\times\Bcal| \cdot \log(1+1/\epsilon) \big).$$
\end{proposition}

\paragraph{Linear function approximation} We consider linear function class $\Fcal$ consisting of functions linear in a $d$-dimensional feature mapping. Specifically, for each $h\in[H]$, we choose $\Fcal_h=\{\phi_h(\cdot,\cdot,\cdot)\trans\theta~\mid~\theta\in B_d(R)\}$ where 
 $\phi_h:\Scal\times\Acal\times\Bcal\rightarrow B_d(1)$ maps a state-action pair into the $d$-dimensional unit ball centered at the origin.

 \begin{assumption}[Self-completeness] \label{as:self-complete}
$\Tcal_{h}^{\mu_f} \cF_{h+1} \subset \cF_h$ for any $(f,h) \in \cF\times[H]$.
\end{assumption}

Assumption \ref{as:self-complete} is a special case of Assumption \ref{as:complete} (completeness) by choosing auxiliary function class $\Gcal=\Fcal$. Assumption \ref{as:self-complete} also implies Assumption \ref{as:realizable} (realizability) by backward induction. Below, we show under self-completeness, the BE dimension of $\Fcal$ is upper bounded by the dimension of the feature mappings $d$ up to a logarithmic factor.


\begin{proposition}[Linear FA $\subset$ Low BE dimension]\label{ex:linear}
Under Assumption \ref{as:self-complete},
 $$\dim_{\textrm{BE}}(\Fcal,\epsilon) \le \Ocal\big( d \cdot \log(1+R/\epsilon) \big).$$
\end{proposition}

We remark that our setting of linear function approximation is strictly more general than linear MGs \citep{xie2020learning} which further assumes both transitions and rewards are linear. The self-completeness assumption is automatically satisfied in linear MGs. Prior works \citep{xie2020learning, chen2021almost} can only address linear MGs or its variant but not the more general setting of linear function approximation presented here. Finally, for linear MGs, by combining Proposition \ref{ex:linear} and Corollary \ref{cor:golf-mg-pac}, we immediately obtain a $\tilde{\Ocal}(H^2d^2/\epsilon^2)$ sample complexity bound for finding $\epsilon$-approximate Nash policies, which improves upon the cubic dependence on dimension $d$ in  \cite{xie2020learning}. 

%

\paragraph{Kernel function approximation} This setting extends linear function approximation from $d$-dimensional Euclidean space $\R^d$ to a decomposable Hilbert space $\cH$. Formally, for each $h\in[H]$, we choose $\Fcal_h=\{\phi_h(\cdot,\cdot,\cdot)\trans\theta~\mid~\theta\in B_\cH(R)\}$ where 
 $\phi_h:\Scal\times\Acal\times\Bcal\rightarrow B_\cH(1)$. Since the ambient dimension of $\cH$ can be infinite, we leverage the notion of \emph{effective dimension} and prove the BE dimension of $\Fcal$ is no larger than the effective dimension of certain feature sets in $\cH$.

\begin{proposition}[Kernel FA $\subset$ Low BE dimension]\label{ex:kernel}
Under Assumption \ref{as:self-complete}, 
 \begin{equation}\label{eq:kernel}
\dim_{\textrm{BE}}(\Fcal,\epsilon) \le \max_{h\in[H]}  d_{\rm{eff}}\big(\Xcal_h,\epsilon/(2R+1)\big),
 \end{equation}
 
 where $\Xcal_h=\{\E_{\mu}[\phi_h(s_h,a_h,b_h)]~\mid~\mu\in\Dcal_{\Fcal,h}\}$ and $d_{\rm{eff}}$ is the effective dimension.
\end{proposition}

The $\epsilon$-effective dimension of a set $\Zcal$  is the minimum integer $d_{{\rm eff}}(\Zcal,\epsilon)=n$ such that

$$\sup_{z_1,\ldots,z_n\in \Zcal}	\frac{1}{n} \log\det\big(\mathrm I + \frac{1}{\epsilon^2} \sum_{i=1}^{n} z_i z_i^\top \big) \le e^{-1}.$$

For $\cH=\R^d$,  the effective dimension in the RHS of \eqref{eq:kernel} is upper bounded by $\tilde\Ocal(d)$ by following the standard ellipsoid potential argument.

\paragraph{MGs with rich observation} In this scenario, despite the MG has a large number of states, these states can be categorized into a small number of ``effective states''. Formally, there exists an unknown decoding function $q:\Scal\rightarrow [m]$  such that if two states $s$ and $s'$ satisfy $q(s)=q(s')$, then their transition measures and reward functions are identical, i.e., 
$\Pr_h(\cdot\mid s,a,b) = \Pr_h(\cdot\mid s',a,b)$, $\Pr_h(s\mid \cdot,a,b) = \Pr_h(s'\mid \cdot,a,b)$,  and $r_h(s,a,b)=r_h(s',a,b)$ for all $(h,a,b)$.
Here, the decoding function $q$ is unknown to the learner, and the size of its codomain $m$ is usually much smaller than that of the  state space.

\begin{proposition}[MGs with rich observation $\subset$ Low BE dimension]\label{ex:block} Consider an MG with decoding function $q:\Scal\rightarrow [m]$. 
Then any function class $\Fcal\subseteq (\Scal\times\Acal\times\Bcal\rightarrow [0,1])$ satisfies 
 $$\dim_{\textrm{VBE}}(\Fcal,\epsilon) \le \Ocal\big( m\cdot \log(1+1/\epsilon) \big),$$
 where $\dim_{\textrm{VBE}}$ denotes the V-type BE dimension defined in Appendix \ref{sec:type2}.
\end{proposition}

\paragraph{Kernel feature selection}
To begin with, we introduce kernel MGs, which is a special case of kernel function approximation introduced earlier in this section by additionally assuming that both transitions and rewards are linear in RKHS. Concretely, in a kernel MG, for each step $h\in[H]$, there exist feature mappings $\phi_h: \Scal\times\Acal\times\Bcal\rightarrow\cH$ and  $\psi_h: \Scal\rightarrow\cH$ where $\cH$ is a decomposable Hilbert space, so that the transition measure can be represented as the inner product of features, i.e.,  $\Pr_h(s'\mid s,a,b)=\langle \phi_h(s,a,b),\psi_h(s')\rangle_\cH$.  Besides, the reward function is linear in  $\phi$, i.e., $r_h(s,a,b)=\langle\phi_h(s,a,b),\theta_h^r\rangle_\cH$ for some  $\theta_h^r\in\cH$. Here, $\phi$, $\psi$ and $\theta^r$ are \emph{unknown} to the learner. Moreover, a kernel MG satisfies the following regularization conditions: for all $h$, (a) 
$\|\theta_h^r\|_\cH\le 1$, and (b) $\| \sum_{s\in\Scal} \Vcal(s)\psi_h({s})\|_\cH \le 1$ for any function $\Vcal:\Scal\rightarrow [0,1]$.

In kernel feature selection, the learner is provided with a feature class $\Phi$ satisfying: (a) $\phi\in\Phi$, and (b) $\|\tilde\phi_h(s,a,b)\|_2\le 1$ for all $(s,a,b,h)$ and $\tilde\phi\in\Phi$. 
 Therefore, it is natural to consider the  union of all linear function classes induced by $\Phi$. Specifically, let $\Fcal:=  \Fcal_1\times\cdots\times\Fcal_H$ where 
	$\Fcal_h=\{\langle \tilde\phi_h(\cdot,\cdot,\cdot),\theta\rangle_\cH ~\mid~\theta\in B_\cH(R),~\tilde\phi\in\Phi\}$. 
Below, we show the V-type BE dimension of $\Fcal$ is upper bounded by the effective dimension of certain feature sets induced by $\Fcal$. 
\begin{proposition}[Kernel feature selection $\subset$ Low BE dimension]\label{ex:selection}
Let $\Mcal$ be a kernel MG. Then 
 $$
\dim_{\textrm{VBE}}(\Fcal,\epsilon) \le \max_{h\in[H]}  d_{\rm{eff}}\big(\Xcal_h,\epsilon/(2R+1)\big),
 $$
 
 where $\Xcal_h=\{\E_{\mu}[\phi_h(s_h,a_h,b_h)]~\mid~\mu\in\Dcal_{\Fcal,h}\}$.
 \end{proposition}


\section{Conclusion}

This paper presents the first line of sample-efficient results for Markov games with large state space and general function approximation. 
We propose a new complexity measure---multiagent Bellman-Eluder dimension, and design a new algorithm that can sample-efficiently learn any MGs with low BE dimension. At the heart of our algorithm is the exploiter, which facilitates the learning of the main player by continuously exploiting her weakness.  
Our generic framework applies to a wide range of new problems including MGs with linear or kernel function approximation, MGs with rich observations, and kernel feature selection, all of which can not be addressed by existing works.



\vspace{5ex}

\bibliographystyle{abbrv}
\bibliography{ref}

\clearpage
\appendix

\section{Discussion on Technical Challenges}
\label{sec:technical}

In this section, we discuss some technical challenges faced when designing provably efficient algorithms for MGs with general function approximation, which explains why direct extension of existing algorithmic solutions is not enough.

The algorithmic solutions designed for tabular MGs \citep{bai2020provable} and linear MGs \citep{xie2020learning} can be viewed as special cases of solving the following sub-problem: in each episode $k$, given the confidence set of possible functions $f \in \Ccal^k$, find function pair $(f^k,g^k)$ and policy pair $(\mu^k,\nu^k)$ s.t. for any state $s$
\begin{equation}
    \label{equ:sub-problem}
	[\D_{\mu ^k\times \nu ^k}]f^k(s) \ge \max_{f\in \Ccal^k,\mu}[\D_{\mu \times \nu ^k} f](s), \,\,\,\,
    [\D_{\mu ^k\times \nu ^k}g^k](s) \le \min_{g\in \Ccal^k,\nu}[\D_{\mu ^k \times \nu} g](s) .    
\end{equation}

Here we omit the dependence on $h$ by considering a single-step special case. This is similar to the contextual bandits problem, but a game version.

Once the above sub-problem is solved, using the fact that the true value function $f^{\star}$ is contained in $\Ccal^k$, we can bound the duality gap by 
$$
(V^{\dagger ,\nu ^k}-V^{\mu ^k,\dagger})(s)=[\D_{\mu^{\dagger}( \nu ^k ) \times \nu ^k}f^\star](s)-[\D_{\mu ^k\times \nu^{\dagger}( \mu ^k ) }f^\star](s)
\le [\D_{\mu ^k\times \nu ^k}(f^k-g^k) ](s),
$$
where the summation of the RHS over $k$ can be further bounded by pigeon-hole type of arguments.

Then a natural question is: how can we solve this sub-problem \eqref{equ:sub-problem}? A helpful condition is optimistic closure, which indeed holds for tabular and linear MGs \citep{bai2020provable,xie2020learning}. For the max-player version, this means the pointwise upper bound defined by $\overline{f}^k\left( s,a,b \right) :=\underset{f\in \mathcal{C}^k}{\max}f\left( s,a,b \right)$ remains in the function class, i.e., $\overline{f}^k \in \cC^k\subseteq \Fcal$. The min-player version is similar.

Under this condition, the function pair $(\up{f}^k,\low{f}^k)$ is clearly the maximizer/minimizer to the subproblem, and therefore it reduces to finding  $(\mu^k,\nu^k)$  so that
$$
\D_{\mu ^k\times \nu ^k}\up{f}^k(s) \ge \max_{\mu}\D_{\mu \times \nu ^k} \up{f}^k(s), \,\,\,\,
    \D_{\mu ^k\times \nu ^k}\low{f}^k(s) \le \min_{\nu}\D_{\mu ^k \times \nu} \low{f}^k(s) .
$$
By definition, the solution is a coarse correlated equilibrium. 

However, the optimistic closure may not hold in general. Even worse, no solution may exist for the sub-problem ~\eqref{equ:sub-problem}, as we will see below in a concrete example. In that case, the existing techniques fail and it becomes unclear how to bound the two-sided duality gap. As a result, the general function approximation problem is significantly harder than the special cases mentioned before.

Now we describe a concrete example where sub-problem \eqref{equ:sub-problem} has no solution when \emph{optimisic closure condition does not hold}. We consider rock-paper-scissors, with one state and three actions for both players, i.e., $|\Scal|=1$ and $|\Acal|=|\Bcal|=3$. Then each function $f \in \Fcal:\Scal \times \Acal \times \Bcal \rightarrow [-2,2]$ \footnote{We assign the reward $[-2,2]$ instead of $[0,1]$ to make the example looks simpler. Everything stated here still hold after scaling.} is essentially a $3 \times 3$ matrix and we will use $M$ to represent such matrices. The matrix corresponding to Nash value $Q^{\star}$ is  
$$
M^{\star}=\left( \begin{matrix}
	0&		1&		-1\\
	-1&		0&		1\\
	1&		-1&		0\\
\end{matrix} \right) 
$$
which describes the true reward associated with different actions. Assume there are $6$ other matrices in our confidence set $\cC^k$:
$$
M_1=\left( \begin{matrix}
	0&		1.1&		-1\\
	-1&		0&		1\\
	1&		-1&		0\\
\end{matrix} \right) ,\,\, M_2=\left( \begin{matrix}
	0&		1&		-1.1\\
	-1&		0&		1\\
	1&		-1&		0\\
\end{matrix} \right) ,\,\, M_3=\left( \begin{matrix}
	0&		1&		-1\\
	-1.1&		0&		1\\
	1&		-1&		0\\
\end{matrix} \right), 
$$
$$
M_4=\left( \begin{matrix}
	0&		1&		-1\\
	-1&		0&		1.1\\
	1&		-1&		0\\
\end{matrix} \right) ,\,\, M_5=\left( \begin{matrix}
	0&		1&		-1\\
	-1&		0&		1\\
	1.1&		-1&		0\\
\end{matrix} \right) ,\,\, M_6=\left( \begin{matrix}
	0&		1&		-1\\
	-1&		0&		1\\
	1&		-1.1&		0\\
\end{matrix} \right). 
$$

Now suppose we find a pair of matrices $(\up{M},\low{M})$ and a pair of policies $(\mu,\nu)$ which solves sub-problem ~\eqref{equ:sub-problem}. If we can prove such $(\mu,\nu)$ must be deterministic, then we get into a contradiction, since whatever $(\up{M},\low{M})$ is, it is impossible for deterministic policies $\mu$ and $\nu$ to be the best response of each other. Therefore, in order to show sub-problem ~\eqref{equ:sub-problem} has no solution, it suffices to  prove $(\mu,\nu)$ must be deterministic. 

Since $\mu$ is a solution of sub-problem ~\eqref{equ:sub-problem}, we must have $$\mu = \argmax_{\mu'}\max_{M \in \cC^k}(\mu')^TM\nu.$$

Since we are maximizing the above quantity, $\up{M}$ can only take $M_1$,$M_4$ or $M_5$ because the others are dominated. Direct computation gives 
\begin{align*}
	M_1\nu =& M^{\star} \nu + [0.1\nu[2],0,0]^T,\\
	M_4\nu =& M^{\star} \nu + [0,0.1\nu[3],0]^T,\\
	M_5\nu =& M^{\star} \nu + [0,0,0.1\nu[1]]^T.
\end{align*}

Now that $\mu$ maximizes $(\mu')^T \up{M}\nu$, if it is not deterministic, there will be at least two entries in $\up{M}\nu$ that take the same highest value, say $(\up{M}\nu)[1] =(\up{M}\nu)[2] \ge  (\up{M}\nu)[3]$. 

If $\up{M} = M_1$, the above condition is just 
$$
(\up{M}\nu)[2] = (M^{\star} \nu)[2] = (M^{\star} \nu)[1] + 0.1\nu[2] = (\up{M}\nu)[1].$$ 

Then consider two cases:
\begin{itemize}
	\item If $\nu[3] >0$, we can choose $\up{M} = M_4$ to make $(\up{M}\nu)[2] = (M^{\star} \nu)[2] +0.1\nu[3] > (M^{\star} \nu)[2]$. Therefore, the value of $\max(\mu')^T\up{M}\nu$ is increased, which is contradictory to the maximization condition.
	\item If $\nu[3] = 0$, the condition $(M^{\star} \nu)[2] = (M^{\star} \nu)[1] + 0.1\nu[2]$ is just $\nu[2]+0.1\nu[2] = -\nu[1]$, which impliess $\nu[2] = \nu[1] = 0$. This is impossible since $\nu$ is a probability distribution.
\end{itemize}

As a result, we prove $\up{M}$ cannot be $M_1$. Similarly, $\up{M}$ cannot be $M_4$ or $M_5$. We obtain a contradiction! Therefore, $\mu$ must be deterministic and by symmtrical argument so is $\nu$. Putting everything together completes the proof.
That is, when \emph{optimisic closure condition does not hold}, sub-problem \eqref{equ:sub-problem} can have no solution.

\section{V-type Bellman Eluder Dimension}
\label{sec:type2}

In this section, we define V-type Bellman-Eluder (VBE) dimension, and provide the corresponding theoretical guarantee for Algorithm \ref{alg:golfmg} with 
Option II.


\subsection{Complexity Measure}

With slight abuse of notation, 
we redefine the following distribution families for VBE dimension. The only difference is that now we use distributions over $\Scal$ instead of $\Scal\times\Acal\times\Bcal$.

\begin{enumerate}
	\item $\Dcal_{\dirac}:=\{\Dcal_{\dirac,h}\}_{h\in[H]}$, where $\Dcal_{\dirac,h} = \{\delta_{s}(\cdot) ~|~ s\in\cS\}$, i.e., the collections of probability measures that put measure $1$ on a single state.
	\item $\Dcal_\cF:=\{\Dcal_{\mu,h}\}_{h\in[H]}$, where $\Dcal_{\mu,h}$ denotes the collections of all probability measures over 
	$\Scal$ at the $h^\text{th}$ step, which can be generated by executing some $(\fmu_f,\fres_{f,g})$ with $f,g\in\Fcal$. Here, $\fres_{f,g}$ is the best response to $\mu_f$ regarding to Q-value $g$:
$$\fres_{f,g,h}(s) : = \argmin_{\nu\in\Delta_\Bcal}\fmu_{f,h}(s)\trans g_h(s,\cdot,\cdot)\nu \quad \mbox{for all } (s,h).$$
\end{enumerate}
To proceed, we introduce the following  Bellman residual function defined over the state space $\Scal$
$$
[\Ecal_h(f,g,w)](s):= \E[(f_h-\Tcal_{h}^{\mu_g} f_{h+1})(s,a,b)~\mid~(a,b)\sim \mu_{g,h} \times\nu_{g,w,h}].
$$

Equipped with the new definition, we can define V-type BE dimension.

\begin{definition}[V-type Bellman-Eluder dimension]
    \label{def:vbedim}
    Let $\cH_\cF$ be the function classes of Bellman residual where $\cH_{\cF, h} := \{\Ecal_h(f,g,w) ~|~ f, g,w\in\Fcal\}$. Then the V-type Bellman-Eluder (VBE) dimension is defined as 
    \begin{equation}
    \dim_{\textrm{VBE}}(\Fcal, \epsilon)= \max_{h\in[H]}\min_{\Dcal\in\{\Dcal_\Delta,\Dcal_\cF\}} \dim_{\textrm{DE}}(\cH_{\cF,h},\Dcal_h,\epsilon).
    \end{equation}
    \end{definition}

  We comment that we do not have the online version of V-type BE dimension, because  the uniform sampling step in Option II is in general not compatible with the policy of the opponent.

    \subsection{Theoretical Guarantees}
Now we are ready to present the theoretical guarantees.

\begin{theorem}[Regret of \golfmg]
\label{thm:golf-mg-regret-V}
Under Assumption \ref{as:realizable} and \ref{as:complete}, 
there exists an absolute constant $c$ such that for any $\delta\in(0,1]$, $K\in\N$, 
if we choose  $\beta=c\cdot(\log(KH|\Fcal||\Gcal|/\delta)+K\CompErr^2+K\RealErr^2)$, then with probability at least $1-\delta$, for all $k\in[K]$, Algorithm \ref{alg:golfmg} with Option II satisfies:
$$
\sum_{t=1}^k V_1^\star(s_1) -V_1^{\mu^t,\dagger}(s_1) \le \mathcal{O}(H\sqrt{|\Acal||\Bcal|dk\beta}),
$$ 
where $d=\dim_{\textrm{VBE}}(\Fcal, 1/K)$ is the VBE dimension.
\end{theorem}

Theorem \ref{thm:golf-mg-regret-V} provides a $\sqrt{K}$ pseudo-regret guarantee   in terms of VBE dimension for Algorithm \ref{alg:golfmg} with Option II. 
The reason for calling it pseudo-regret is that in each episode, the samples are collected following a combination of $(\mu^k,\nu^k)$ and uniform sampling, instead of only $(\mu^k,\nu^k)$ as in Option I. Still, this is enough for applying the pigeonhole principle to derive the sample complexity of \golfmg.

\begin{corollary}[Sample complexity of \golfmg]
	\label{cor:golf-mg-pac-V}
Under Assumption \ref{as:realizable} and \ref{as:complete}, 
there exists an absolute constant $c$ such that for any $\epsilon>0$, choose $\beta=c\cdot(\log(KH|\Fcal||\Gcal|/\delta)+K\CompErr^2+K\RealErr^2)$ and $\Delta = c'(H\sqrt{|\Acal||\Bcal|d\beta/K}+\epsilon)$, where $d=\dim_{\textrm{VBE}}(\Fcal, \epsilon/H)$ is the VBE dimension of $\Fcal$, then with probability at least $1-\delta$, the output condition (Line~\ref{line:output-condition}) will be satistied at least once in the first $K$ episodes. Furthermore, 
the output policy $\mu^{\text{out}}$ of Algorithm \ref{alg:golfmg} with Option II is $\Ocal(\epsilon+H\sqrt{d}(\RealErr+\CompErr))$-approximate Nash, if  $K \ge \Omega\big( (H^2d/\epsilon^2)\cdot\log(H|\Fcal||\Gcal|d/\epsilon)\big)$
.
\end{corollary}
Corollary \ref{cor:golf-mg-pac-V} claims that 
\golfmg\ can find an $\Ocal(\epsilon)$-approximate Nash policy for the max-player using at most 
$\tilde\Ocal(H^2|\Acal||\Bcal|\log(|\Fcal||\Gcal|)\dim_{\textrm{BE}}(\Fcal, \epsilon)/\epsilon^2)$ samples, which scales linearly in the VBE dimension, the cardinality of the two action sets, and the log-cardinality of the function classes.

\section{Proofs for Section~\ref{sec:main-results}}
\label{app:typeI-golf}

In this section, we present the proof of Theorem~\ref{thm:golf-mg-regret}, Corollary \ref{cor:golf-mg-pac}, and Theorem~\ref{thm:golf-online}. 

The following auxiliary lemma (Lemma 26 in \cite{jin2021bellman}) will be useful.

\begin{lemma}
	[Pigeon-hole]
	\label{lem:de-regret}
	Given a function class $\Gcal$ defined on $\Xcal$ with $|g(x)|\le C$ for all $(g,x)\in\Gcal\times\Xcal$, and a family of probability measures $\Pi$ over $\Xcal$. 
		Suppose sequence $\{g_k\}_{k=1}^{K}\subset \Gcal$ and $\{\mu_k\}_{k=1}^{K}\subset\Pi$ satisfy that for all $k\in[K]$,
		$\sum_{t=1}^{k-1} (\E_{\mu_t} [g_k])^2 \le \beta$. Then for all $k\in[K]$ and $\omega>0$,
		$$
		\sum_{t=1}^{k} |\E_{\mu_t} [g_t]| \le \Ocal\left(\sqrt{\dedim (\Gcal,\Pi,\omega)\beta k}+\min\{k,\dedim (\Gcal,\Pi,\omega)\}C +k\omega\right).
		$$
	\end{lemma}	

Another useful decomposition is the value difference lemma (Lemma 1 in \cite{jiang2017contextual}).

\begin{lemma}[value difference]
	\label{lem:value-difference}
	For any function $f$ s.t. $f_{H+1}=0$ and policy $\pi$,
	$$
V_{f,h}(s) - V_h^{\pi}(s)
=\sum_{h'=h}^H
\E_{\pi}\left[V_{f,h'}(s_{h'}) - 
r_{h'}(s_{h'},a_{h'},b_{h'}) - V_{f,h'+1}(s_{h'+1})|s_h=s\right].
$$
\end{lemma}

\subsection{Proof of the Online  Guarantee}

We begin with the proof of Theorem~\ref{thm:golf-online}. The following two concentration lemmas help us upper bound the Bellman residuals. We defer their proofs to Appendix~\ref{sec:concentrate-1} and \ref{sec:concentrate-2}.

\begin{lemma}\label{lem:concentrate-1}
Assuming  $\|Q_h^\star-\pfh(Q_h^\star)\| \le \RealErr$ and $\|\Tcal_{h} f_{h+1} -\pgh(\Tcal_{h} f_{h+1})\| \le \CompErr$ for all $f \in \cF$ and $h\in[H]$, if we choose $\beta=c\cdot(\log(KH|\Fcal||\Gcal|/\delta)+K\CompErr^2+K\RealErr^2)$ with some large absolute constant $c$ in Algorithm \ref{alg:golfmg} with Option I, then
	with probability at least $1-\delta$, for all $(k,h)\in[K]\times[H]$, we have 
	\begin{enumerate}[label=(\alph*)]
		\item $\sum_{i=1}^{k-1}  \E \left[ \paren{f_h^k(s_h,a_h,b_h) - (\cT_h f_{h+1}^k)(s_h,a_h,b_h) }^2\mid s_h,a_h,b_h\sim \pi^i\right]
		{\le} \mathcal{O} ( \beta)$,
		\item $\sum_{i=1}^{k-1}  \paren{f^k_h(s_h^i,a_h^i,b_h^i) - (\cT_h f_{h+1}^k)(s_h^i,a_h^i,b_h^i) }^2
		{\le} \mathcal{O} ( \beta)$,
	\end{enumerate}
	where $(s_1^i,a_1^i,,b_1^i,\ldots, s_H^i,a_H^i,b_H^i)$ denotes the trajectory  sampled by following $\pi^i = (\mu^i,\nu^i)$ in the $i^{\rm th}$ episode. 
\end{lemma}

\begin{lemma}\label{lem:concentrate-2}
Under the same condition of Lemma \ref{lem:concentrate-1}, with probability at least $1-\delta$, we have 
	 $\pf(Q^\star)\in \Ccal^k$ for all $k\in[K]$.
\end{lemma}

\begin{proof}[Proof of Theorem~\ref{thm:golf-online}]
By Lemma \ref{lem:concentrate-2}, $\pf(Q^\star)\in \Ccal^k$. Since we choose  $f^k$ optimistically and $\Fcal$ satisfies $\RealErr$-approximate realizability,
\begin{align*}
	&\sum_{k=1}^K \left(V^{\star}_1(s_1) - V_1^{\mu^k,\nu^k}(s_1)\right)\\
	\le& 	\sum_{k=1}^K
	\left(V_{ \pf(Q^\star),1}(s_1) - V_1^{\mu^k,\nu^k}(s_1)\right)+K\RealErr \\
	\le& 	\sum_{k=1}^K
	\left(V_{f^k,1}(s_1) - V_1^{\mu^k,\nu^k}(s_1)\right)+K\RealErr\\
	\overset{\left( i \right)}{=} & \sum_{k=1} ^K\sum_{h=1}^H
\E_{\pi^k}\left[V_{f^k,h}(s_h) - 
r_h(s_h,a_h,b_h) - V_{f^k,h+1}(s_{h+1})\right]+K\RealErr \\
\overset{\left( ii \right)}{=} & \sum_{k=1} ^K\sum_{h=1}^H
\E_{\pi^k}\left[\min_{\nu}\D_{\mu^k_h \times \nu}
	f_h^k(s_h) - 
r_h(s_h,a_h,b_h) - V_{f^k,h+1}(s_{h+1})\right]+K\RealErr  \\
\le & \sum_{k=1} ^K\sum_{h=1}^H
\E_{\pi^k}\left[\D_{\mu^k_h \times \nu^k_h}
f_h^k(s_h) - 
r_h(s_h,a_h,b_h) - V_{f^k,h+1}(s_{h+1})\right]+K\RealErr \\
= & \sum_{h=1}^H\sum_{k=1} ^K
\E_{\pi^k}\left[ f_h^k(s_h,a_h,b_h) - 
r_h(s_h,a_h,b_h) - V_{f^k,h+1}(s_{h+1})\right]+K\RealErr \\
\overset{\left( iii \right)}{\le} & \sum_{h=1}^H \sum_{k=1}^K  (f^k_h-\Tcal_h f_{h+1}^k)(s_h^k,a_h^k,b_h^k)+\Ocal\left(\sqrt{KH\log(KH/\delta)}\right)+K\RealErr \\
\overset{\left( iv \right)}{\le} & \Ocal\left(H\sqrt{\beta \cdot \dim_{\textrm{OBE}}(\Fcal,1/K) \cdot K}\right)+K\RealErr \le \Ocal\left(H\sqrt{\beta \cdot \dim_{\textrm{OBE}}(\Fcal,1/K) \cdot K}\right).
\end{align*}
where $(i)$ is by value difference (Lemma~\ref{lem:value-difference}), 
$(ii)$ is due to $\mu^k=\mu_{f^k}$, 
$(iii)$ is by Azuma-Hoeffding, and $(iv)$ is by incurring Lemma \ref{lem:de-regret} with $\Gcal = \cH_{\cF,h}$ (here, $\cH_\Fcal$  refers to the one introduced  in Definition \ref{def:obe}), $\Pi = \Dcal_{\Delta,h}$, $\epsilon  = 1/K$ and Lemma~\ref{lem:concentrate-1}. 
The final inequality follows from the choice of $\beta$.
\end{proof}

\subsection{Proof of the Self-play Guarantee}

Proving Theorem~\ref{thm:golf-mg-regret} requires more work because we need to develop guarantees for the sub-routine \golfbr. Similar to Lemma~\ref{lem:concentrate-1} and \ref{lem:concentrate-2}, we establish two concentration lemmas below, whose proofs are deferred to Section~\ref{sec:concentrate-1-routine} and \ref{sec:concentrate-2-routine}.

\begin{lemma}
\label{lem:concentrate-1-routine}
Under Assumption \ref{as:realizable} and \ref{as:complete}, if we choose $\beta=c\cdot(\log(KH|\Fcal||\Gcal|/\delta)+K\CompErr^2+K\RealErr^2)$ with some large absolute constant $c$ in Algorithm \ref{alg:golfmg} with Option I, then
	with probability at least $1-\delta$, for all $(k,h)\in[K]\times[H]$, we have 
	\begin{enumerate}[label=(\alph*)]
		\item $\sum_{i=1}^{k-1}  \E \left[ \paren{\tilde{f}_h^k(s_h,a_h,b_h) - (\cT^{\mu^k}_h \tilde{f}_{h+1}^k)(s_h,a_h,b_h) }^2\mid s_h,a_h,b_h\sim \pi^i\right]
		{\le} \mathcal{O} ( \beta)$.
		\item $\sum_{i=1}^{k-1}  \paren{\tilde{f}^k_h(s_h^i,a_h^i,b_h^i) - (\cT^{\mu^k}_h \tilde{f}_{h+1}^k)(s_h^i,a_h^i,b_h^i) }^2
		{\le} \mathcal{O} ( \beta)$,
	\end{enumerate}
	where $(s_1^i,a_1^i,,b_1^i,\ldots, s_H^i,a_H^i,b_H^i,s_{H+1}^i)$ denotes the trajectory  sampled by following $\pi^i = (\mu^i,\nu^i)$ in the $i^{\rm th}$ episode, and $\tilde f^k$ denotes the optimistic function computed in sub-routine \golfbr\ in the  $k^{\rm th}$ episode.
\end{lemma}

\begin{lemma}\label{lem:concentrate-2-routine}
Under the same condition of Lemma \ref{lem:concentrate-1-routine}, with probability at least $1-\delta$, we have 
	 $\pf(Q^{\mu^k,\dagger})\in \cC^{\mu^k}$ for all $k\in[K]$.
\end{lemma}

\begin{proposition}[approximate best-response regret]
\label{prop:routine}
Under Assumption \ref{as:realizable} and \ref{as:complete}, 
there exists an absolute constant $c$ such that for any $\delta\in(0,1]$, $K\in\N$, 
if we choose $\beta=c\cdot(\log(KH|\Fcal||\Gcal|/\delta)+K\CompErr^2+K\RealErr^2)$, then
with probability at least $1-\delta$,
for all $k\in[K]$,
$$
\sum_{t=1}^k \left(V^{\mu^t,\nu^t}_1(s_1) -  V^{\mu^t,\dagger}_1(s_1)\right)
\le \Ocal(  H\sqrt{k\beta\cdot \dim_{\textrm{BE}}(\Fcal, 1/K)}).
$$
\end{proposition}

\begin{proof}
By Lemma \ref{lem:concentrate-2-routine}, $\pf(Q^{\mu^t,\dagger})\in \Ccal^t$. Since we choose  $f^t$ optimistically and $\Fcal$ satisfies $\RealErr$-approximate realizability,
\begin{equation}\label{eq:jun1}
	\begin{aligned}
&\quad \sum_{t=1}^k \left(V^{\mu^t,\nu^t}_1(s_1) -  V^{\mu^t,\dagger}_1(s_1)\right) \\
&\le \sum_{t=1}^k \left(V^{\mu^t,\nu^t}_1(s_1) - \min_{\nu}\D_{\mu^t_{1}\times \nu} \tilde{f}_{1}^t(s_1)\right)+k\RealErr \\
& = -\sum_{t=1}^k\sum_{h=1}^H \E_{\pi^t}\left[( \tilde{f}^t_h - \Tcal^{\mu^t}_h\tilde{f}_{h+1}^t)(s_h,a_h,b_h)\right]+k\RealErr \\
&\le - \sum_{h=1}^H \sum_{t=1}^k ( \tilde{f}^t_h-\Tcal^{\mu^t}_h \tilde{f}_{h+1}^t)(s_h^t,a_h^t,b_h^t)+\Ocal(\sqrt{kH\log(KH/\delta)})+k\RealErr ,
\end{aligned}
\end{equation}
 where $\tilde f^t$ denotes the optimistic function computed in sub-routine \golfbr\ in the  $t^{\rm th}$ episode and the equality follows from value difference Lemma \ref{lem:value-difference}.

By Lemma \ref{lem:concentrate-1-routine} (b), we have for all $k\in[K]$
$$
\sum_{t=1}^{k-1} \left[ \tilde{f}^k_h(s_h^t,a_h^t,b_h^t) -  (\Tcal^{\mu^k}_h \tilde{f}_{h+1}^k)(s_h^t,a_h^t,b_h^t)\right]^2\le \Ocal(\beta).
$$
So we can apply Lemma \ref{lem:de-regret} with $\Gcal = \cH_{\cF,h}$, $\Pi = \Dcal_{\Delta,h}$, $\epsilon  = 1/K$ and obtain

$$
\sum_{t=1}^k \left| \tilde{f}^t_h(s_h^t,a_h^t,b_h^t) -  (\Tcal^{\mu^t}_h \tilde{f}_{h+1}^t)(s_h^t,a_h^t,b_h^t)\right|
\le \sqrt{\beta \cdot \dim_{\textrm{DE}}(\cH_{\cF,h},\Dcal_{\Delta,h},1/K) \cdot k}. 
$$
Plugging this inequality back into \eqref{eq:jun1} gives us the first upper bound.

We can also invoke Lemma \ref{lem:concentrate-1-routine} (a) with Jensen's inequality, and obtain  
for all $k\in[K]$
$$\sum_{t=1}^{k-1} \left( \E_{\pi^t} \left[ (\tilde{f}_h^k- \cT^{\mu^k}_h \tilde{f}_{h+1}^k)(s_h,a_h,b_h)\right]\right)^2
		{\le} \mathcal{O} ( \beta).
$$
Again, we can apply Lemma \ref{lem:de-regret} with $\Gcal = \cH_{\cF,h}$ (here, $\cH_\Fcal$  refers to the one introduced   in Definition \ref{def:bedim}), $\Pi = \Dcal_{\cF,h}$, $\epsilon  = 1/K$, and obtain
$$
\sum_{t=1}^k \left|\E_{\pi^t} \left[ (\tilde{f}^t_h-\Tcal^{\mu^t}_h \tilde{f}_{h+1}^t)(s_h^t,a_h^t,b_h^t)\right]\right|
\le \sqrt{\beta \cdot \dim_{\textrm{DE}}(\cH_{\cF,h},\Dcal_{\Fcal,h},1/K) \cdot k}. 
$$
Plugging this inequality back into \eqref{eq:jun1} gives us the second upper bound.

Combining the two upper bounds and noticing that $k\RealErr \le   \Ocal(  H\sqrt{k\beta\cdot \dim_{\textrm{BE}}(\Fcal, 1/K)})$ (because of the choice of $\beta$)   conclude the proof.
\end{proof}

Now we are ready to prove the regret  guarantee. 

\begin{proof}[Proof of Theorem~\ref{thm:golf-mg-regret}]
The regret can be decomposed into two terms
\begin{align*}
\sum_{k=1}^K \left(V^{\star}_1(s_1) - V_1^{\mu^k,\dagger}(s_1)\right)
= \underset{\left( A \right)}{\underbrace{\sum_{k=1}^K \left(V^{\star}_1(s_1) - V_1^{\mu^k,\nu^k}(s_1) \right)}} + \underset{\left( B \right)}{\underbrace{\sum_{k=1}^K \left(V^{\mu^k,\nu^k}_1(s_1) -  V^{\mu^k,\dagger}_1(s_1)\right)}}.
\end{align*}

By Theorem \ref{thm:golf-online}, where the opponent's policy $\{\nu^k\}_{k=1}^K$ can be arbitrary, $(A)$ can be upper bounded by
$$
\sum_{k=1}^K V^{\star}_1(s_1) - V_1^{\pi^k}(s_1)\le \Ocal\left(  H\sqrt{K\beta \cdot \dim_{\textrm{BE}}(\Fcal, 1/K)}\right).
$$

By Proposition \ref{prop:routine}, $(B)$ can be upper bounded by 
	$$
		\sum_{k=1}^K \left(V^{\mu^k,\nu^k}_1(s_1) -  V^{\mu^k,\dagger}_1(s_1)\right) \le \Ocal\left( 
	H\sqrt{K\beta \cdot \dim_{\textrm{BE}}(\Fcal, 1/K)}\right).
$$

Combining the two inequalities concludes the proof.
\end{proof}

By the standard online-to-batch reduction, we can also derive the sample complexity guarantee.

\begin{proof}[Proof of Corollary~\ref{cor:golf-mg-pac}]
We proceed as in the proof of Theorem~\ref{thm:golf-mg-regret} but take $\omega = \epsilon/H$ (instead of $\omega = 1/K$) every time we incur Lemma~\ref{lem:de-regret}. 

Theorem~\ref{thm:golf-online} essentially shows
$$
\sum_{k=1}^K \left(V^{\star}_1(s_1) - V_1^{\mu^k,\nu^k}(s_1)\right) \le \sum_{k=1}^K \left(V_{f^k,1}(s_1) - V_1^{\mu^k,\nu^k}(s_1)\right)
\le \Ocal(  H\sqrt{K\beta\cdot \dim_{\textrm{BE}}(\Fcal, \epsilon/H)}+K\epsilon),
$$
where we have used the fact that $\dim_{\textrm{OBE}}(\Fcal, \epsilon/H) \le \dim_{\textrm{BE}}(\Fcal, \epsilon/H)$.

Proposition~\ref{prop:routine} essentially shows
$$
\sum_{k=1}^K \left(V^{\mu^k,\nu^k}_1(s_1) -  V^{\mu^k,\dagger}_1(s_1)\right) \le \sum_{k=1}^K \left(V^{\mu^k,\nu^k}_1(s_1) -  \min_{\nu}\D_{\mu^k_{1}\times \nu} \tilde{f}_{1}^k(s_1)\right)
\le \Ocal(  H\sqrt{K\beta\cdot \dim_{\textrm{BE}}(\Fcal, \epsilon/H)}+K\epsilon).
$$

Comparing with the form in Theorem~\ref{thm:golf-mg-regret}, now we have an additional $\Ocal(K\epsilon)$ term, since we take $\omega = \epsilon/H$ when incurring Lemma~\ref{lem:de-regret}.

Putting them together and noticing by definition $\up{V}^k = V_{f^k,1}(s_1)$ and $\low{V}^k = \min_{\nu}\D_{\mu^k_{1}\times \nu} \tilde{f}_{1}^k(s_1)$, we can get 
$$
\frac{1}{K}\sum_{k=1}^K[\up{V}^k-\low{V}^k] \le \mathcal{O}(\frac{H}{K}\sqrt{\dim_{\textrm{BE}}(\Fcal, \epsilon/H)K\beta}+\epsilon),
$$
with probability at least $1-\delta$, where $\beta=c\cdot(\log(KH|\Fcal||\Gcal|/\delta)+K\CompErr^2+K\RealErr^2)$.

By pigeonhole prinple, there must exist some $k$ s.t. $\up{V}^k-\low{V}^k \le \Delta = c'(H\sqrt{\dim_{\textrm{BE}}(\Fcal, \epsilon/H)\beta/K}+\epsilon)$. Therefore, the output condition must be satisfied by some $k \in [K]$.

To make the right hand side order $\Ocal(\epsilon+ H\sqrt{d}(\RealErr+\CompErr) )$, it suffices to take 
$$
K \ge \Omega\big( (H^2d/\epsilon^2)\cdot\log(H|\Fcal||\Gcal|d/\epsilon)\big),
$$
where $d=\dim_{\textrm{BE}}(\Fcal, \epsilon/H)$.
\end{proof}

\subsection{Proofs of the Concentration Arguments}

\subsubsection{Proof of Lemma \ref{lem:concentrate-1}}
\label{sec:concentrate-1}

 \begin{proof}
We prove inequality $(b)$ first. 

Consider a fixed $(k,h,f)$ tuple. For notational simplicity, we denote $z_h^t:=(s_h^t,a_h^t,b_h^t)$. Let 
\begin{align*}
X_t(h,f) := &\left(f_h(z_h^t)- r_h^t -V_{f,h+1}(s_{h+1}^t)\right)^2 - \left(\proj_\Gcal(\Tcal_h f_{h+1})(z_h^t)- r_h^t -V_{f,h+1}(s_{h+1}^t)\right)^2	
\end{align*}
and $\Ffrak_{t,h}$ be the filtration induced by 
$\{z_1^i,r_1^i,\ldots,z_H^i,r_H^i\}_{i=1}^{t-1}\bigcup\{z_1^t,r_1^t,\ldots,z_h^t\}.$
 We have 
\begin{align*}
	\E [X_t(h,f) \mid \Ffrak_{t,h}] = &\left[(f_h-\Tcal_h f_{h+1})(z_h^t)\right]^2- \left[(\proj_\Gcal(\Tcal_h f_{h+1})-\Tcal_h f_{h+1})(z_h^t)\right]^2 \\
	\ge & \left[(f_h-\Tcal_h f_{h+1})(z_h^t)\right]^2-\CompErr^2
\end{align*}
and
\begin{align*}
\Var[X_t(h,f)\mid \Ffrak_{t,h}]& \le \E [(X_t(h,f))^2 \mid \Ffrak_{t,h}] \le 36[(\proj_\Gcal(\Tcal_h f_{h+1})-f_{h})(z_h^t)]^2\\
&\le \Ocal\left( [(\proj_\Gcal(\Tcal_h f_{h+1})-\Tcal_h f_{h+1})(z_h^t)]^2+[(\Tcal_h f_{h+1}-f_{h})(z_h^t)]^2\right)\\
& \le\Ocal\left( [(\Tcal_h f_{h+1}-f_{h})(z_h^t)]^2+\CompErr^2\right).
\end{align*} 
By Freedman's inequality, we have, with probability at least $1-\delta$, 
\begin{align*}
&  \sum_{t=1}^k\left[(f_h-\Tcal_h f_{h+1})(z_h^t)\right]^2-k\CompErr^2 -\sum_{t=1}^{k} X_t(h,f) \\
\le &
\Ocal \paren{\sqrt{\log(1/\delta)\left(\sum_{t=1}^k [(\Tcal_h f_{h+1}-f_{h})(z_h^t)]^2+k\CompErr^2\right)} + \log(1/\delta)}.
\end{align*}

Now taking a union bound for all $(k,h,f)\in[K]\times[H]\times\Fcal$, we obtain that with probability at least $1-\delta$,  for  all $(k,h,f)\in[K]\times[H]\times\Fcal$
 \begin{equation}
 \label{eq:Nov27-1}
\begin{aligned}
&  \sum_{t=1}^k\left[(f_h-\Tcal_h f_{h+1})(z_h^t)\right]^2-k\CompErr^2 -\sum_{t=1}^{k} X_t(h,f) \\
\le &
\Ocal \paren{\sqrt{\iota\left(\sum_{t=1}^k [(\Tcal_h f_{h+1}-f_{h})(z_h^t)]^2+k\CompErr^2\right)} + \iota}.
\end{aligned}
\end{equation}
where  $\iota=\log(HK|\Fcal|/\delta)$.
From now on, we will do all the analysis conditioning on this event being true. 

Consider an arbitrary pair $(h,k)\in[H]\times[K]$. 
By the definition of $\Ccal^{k}$ and Assumption \ref{as:complete}
\begin{equation}\label{eq:Jan25-1}
\begin{aligned}
&	\sum_{t=1}^{k-1} X_t(h,f^{k})\\
=&  \sum_{t=1}^{k-1}(f_h^k(z_h^t)- r_h^t -V_{f^{k},h+1}(s_{h+1}^t))^2 - \sum_{t=1}^{k-1}(\pg(\Tcal_h f^k_{h+1})(z_h^t)- r_h^t -V_{f^{k},h+1}(s_{h+1}^t))^2	\\
 \le&   \sum_{t=1}^{k-1}(f_h^k(z_h^t)- r_h^t -V_{f^{k},h+1}(s_{h+1}^t))^2 - \inf_{g\in\Gcal_{h}}\sum_{t=1}^{k-1}(g(z_h^t)- r_h^t -V_{f^{k},h+1}(s_{h+1}^t))^2
 \le\beta.
\end{aligned}
\end{equation}

Putting \eqref{eq:Nov27-1} and \eqref{eq:Jan25-1} together, we obtain
\begin{equation*}
\sum_{t=1}^{k-1} [(f_h^k-\Tcal_h f^k_{h+1})(z_h^t)]^2 \le \Ocal(\beta + \iota+k\CompErr^2 ),
\end{equation*}
which concludes the proof of inequality $(b)$.

To prove inequality $(a)$, we only need to redefine $\Ffrak_{t,h}$ to be the filtration induced by \\
$\{z_1^i,r_1^i,\ldots,z_H^i,r_H^i\}_{i=1}^{t-1}$
and then repeat the arguments above verbatim.
\end{proof}

\subsubsection{Proof of Lemma \ref{lem:concentrate-2}}
\label{sec:concentrate-2}
 \begin{proof}
 We will continue to use the notations defined in the proof of Lemma \ref{lem:concentrate-1}.
 To further simplify notations, we denote
 $$
 \hat Q := \pf(Q^\star) \quad \mbox{and}
 $$

Consider a fixed tuple $(k,h,g)\in[K]\times[H]\times\Gcal$. Let 
\begin{align*}
W_t(h,g) := (g_h(z_h^t)- r_h^t -V_{\hat Q, h+1}(s_{h+1}^t))^2 
- (\hat Q_h(z_h^t)- r_h^t -V_{\hat Q, h+1}(s_{h+1}^t))^2
\end{align*}
and $\Ffrak_{t,h}$ be the filtration induced by 
$\{z_1^i,r_1^i,\ldots,z_H^i,r_H^i\}_{i=1}^{t-1}\bigcup\{z_1^t,r_1^t,\ldots,z_h^t\}.$
 We have 
\begin{align*}
	\E [W_t(h,g) \mid \Ffrak_{t,h}] = 
	&\left[(g_h-\Tcal_h \hat Q_{h+1})(z_h^t)\right]^2- \left[( 
\hat Q_{h}-\Tcal_h \hat Q_{h+1})(z_h^t)\right]^2 \\
\ge & \left[(g_h-\Tcal_h \hat Q_{h+1})(z_h^t)\right]^2 -\RealErr^2
\end{align*}
and
\begin{align*}
\Var[X_t(h,g)\mid \Ffrak_{t,h}]& \le \E [(X_t(h,g))^2 \mid \Ffrak_{t,h}] \le 36[(\hat Q_{h}-g_{h})(z_h^t)]^2\\
&\le \Ocal\left( [(\hat Q_h-\Tcal_h \hat Q_{h+1})(z_h^t)]^2+[(\Tcal_h \hat Q_{h+1}-g_{h})(z_h^t)]^2\right)\\
& \le\Ocal\left( [(\Tcal_h \hat Q_{h+1}-g_{h})(z_h^t)]^2+\CompErr^2\right).
\end{align*} 
By Freedman's inequality, we have, with probability at least $1-\delta$, 
\begin{align*}
& - \sum_{t=1}^{k} W_t(h,g) + \sum_{t=1}^k\left[(g_h-\Tcal_h \hat Q_{h+1})(z_h^t)\right]^2 -k\RealErr^2\\
  \le & 
\Ocal \paren{\sqrt{\log(1/\delta)\left(\sum_{t=1}^k [(\Tcal_h 
\hat Q_{h+1}-g_{h})(z_h^t)]^2+k\CompErr^2\right)} + \log(1/\delta)}.
\end{align*}

Now taking a union bound for all $(k,h,g)\in[K]\times[H]\times\Gcal$, we obtain that with probability at least $1-\delta$,  for  all $(k,h,g)\in[K]\times[H]\times\Gcal$
 \begin{equation}
 \label{eq:May17-1-1}
\begin{aligned}
& - \sum_{t=1}^{k} W_t(h,g) + \sum_{t=1}^k\left[(g_h-\Tcal_h \hat Q_{h+1})(z_h^t)\right]^2 -k\RealErr^2\\
  \le & 
\Ocal \paren{\sqrt{\iota \left(\sum_{t=1}^k [(\Tcal_h 
\hat Q_{h+1}-g_{h})(z_h^t)]^2+k\CompErr^2\right)} + \iota},
\end{aligned}
\end{equation}
where  $\iota=\log(HK|\Gcal|/\delta)$.
From now on, we will do all the analysis conditioning on this event being true. 

Since $\sum_{t=1}^k \left[(g_h-\Tcal_h \hat Q_{h+1})(z_h^t)\right]^2$ is nonnegative, by Cauchy-Schwarz inequality, \eqref{eq:May17-1-1} implies for  all $(k,h,g)\in[K]\times[H]\times\Gcal$
$$
\sum_{t=1}^{k} W_t(h,g) \ge -\Ocal(\iota+k\CompErr^2).
$$
Plugging in the definition of $W_t(h,g)$ and the choice of $\beta$ completes the proof.
\end{proof}

\subsubsection{Proof of Lemma \ref{lem:concentrate-1-routine}}
\label{sec:concentrate-1-routine}
 \begin{proof}
 Recall $\mu_f$ denotes the Nash policy of the max-player induced by $f$. If there exist more than one induced Nash policies, we can break the tie arbitrarily so that $\mu_f$ is uniquely defined for each $f\in\Fcal$. 
 Denote $\Pi_\Fcal:=\{\mu_f~|~f \in \Fcal\}$. We have $|\Pi_\Fcal|\le |\Fcal|$.

We prove inequality $(b)$ first. 	Consider a fixed tuple $(k,h,f,\mu) \in[K]\times[H]\times\Fcal\times\Pi_\Fcal$. Again we denote $z_h^t:=(s_h^t,a_h^t,b_h^t)$. Let 
\begin{align*}
X_t(h,f,\mu) := \left(f_h(z_h^t)- r_h^t -V^{\mu}_{f,h+1}  (s_{h+1}^t)\right)^2 - \left(\pg(\Tcal^{\mu}_{h}f_{h+1})(z_h^t)- r_h^t -V^{\mu}_{f,h+1}  (s_{h+1}^t)\right)^2	
\end{align*}
and $\Ffrak_{t,h}$ be the filtration induced by 
$\{z_1^i,r_1^i,\ldots,z_H^i,r_H^i\}_{i=1}^{t-1}\bigcup\{z_1^t,r_1^t,\ldots,z_h^t\}.$
 We have 
\begin{align*}
\E [X_t(h,f,\mu) \mid \Ffrak_{t,h}] &= \left[(f_h-\Tcal^{\mu}_{h}f_{h+1})(z_h^t)\right]^2-\left[(\pg(\Tcal^{\mu}_{h}f_{h+1})-\Tcal^{\mu}_{h}f_{h+1})(z_h^t)\right]^2\\
& \ge \left[(f_h-\Tcal^{\mu}_{h}f_{h+1})(z_h^t)\right]^2 - \CompErr^2
\end{align*}
and by Cauchy-Schwarz inequality  
\begin{align*}
&\Var[X_t(h,f,\mu)\mid \Ffrak_{t,h}]\le \E [(X_t(h,f,\mu))^2 \mid \Ffrak_{t,h}]\\
 \le& 36\left[(f_h-\pg(\Tcal^{\mu}_{h}f_{h+1}))(z_h^t)\right]^2
 \le \Ocal\left(\left[(f_h-\Tcal^{\mu}_{h}f_{h+1})(z_h^t)\right]^2+ \CompErr^2\right).
\end{align*} 
By Freedman's inequality, we have, with probability at least $1-\delta$, 
\begin{align*}
&  \sum_{t=1}^k\left[(f_h-\Tcal^{\mu}_{h}f_{h+1})(z_h^t)\right]^2 - k\CompErr^2 - \sum_{t=1}^{k} X_t(h,f,\mu) \\
 \le &
\Ocal \paren{ \sqrt{\log(1/\delta)\sum_{t=1}^k\left(\left[(f_h-\Tcal^{\mu}_{h}f_{h+1})(z_h^t)\right]^2+ \CompErr^2\right)  + \log(1/\delta)}}.
\end{align*}

Now taking a union bound for all $(k,h,f,\mu)\in[K]\times[H]\times\Fcal\times\Pi_\Fcal$, we obtain that with probability at least $1-\delta$,  for  all $(k,h,f,\mu)\in[K]\times[H]\times\Fcal\times\Pi_\Fcal$
 \begin{equation}
 \label{eq:May18-2}
\begin{aligned}
&  \sum_{t=1}^k\left[(f_h-\Tcal^{\mu}_{h}f_{h+1})(z_h^t)\right]^2 - k\CompErr^2 - \sum_{t=1}^{k} X_t(h,f,\mu) \\
 \le &
\Ocal \paren{ \sqrt{\iota\sum_{t=1}^k\left(\left[(f_h-\Tcal^{\mu}_{h}f_{h+1})(z_h^t)\right]^2+ \CompErr^2\right)  + \iota}},
\end{aligned}
\end{equation}
where  $\iota=\log(HK|\Fcal|/\delta)$.
From now on, we will do all the analysis conditioning on this event being true. 

Consider an arbitrary pair $(h,k)\in[H]\times[K]$. 
By the definition of $\cC^{k}$ and Assumption \ref{as:complete}
\begin{equation}\label{eq:May18-3}
\begin{aligned}
&	\sum_{t=1}^{k-1} X_t(h,f^{k},\mu^k)\\
=&  \sum_{t=1}^{k-1}(f_h^k(z_h^t)- r_h^t -V^{\mu^k}_{f^k,h+1}  (s_{h+1}^t))^2 - \sum_{t=1}^{k-1}(\pg(\Tcal^{\mu^k}_{h}f^k_{h+1})(z_h^t)- r_h^t -V^{\mu^k}_{f^k,h+1}  (s_{h+1}^t))^2	\\
 \le &  \sum_{t=1}^{k-1}(f_h^k(z_h^t)- r_h^t -V^{\mu^k}_{f^k,h+1}  (s_{h+1}^t))^2 - \inf_{g\in\Gcal_{h}}\sum_{t=1}^{k-1}(g(z_h^t)- r_h^t -V^{\mu^k}_{f^k,h+1}  (s_{h+1}^t))^2\le\beta.
\end{aligned}
\end{equation}

Putting \eqref{eq:May18-2} and \eqref{eq:May18-3} together, we obtain
\begin{equation*}
\sum_{t=1}^{k-1} [(f_h^k-\Tcal^{\mu^k}_{h}f^k_{h+1})(z_h^t)]^2 \le \Ocal(\beta + \iota+k\CompErr^2 ),
\end{equation*}
which concludes the proof of inequality $(b)$.

To prove inequality $(a)$, we only need to redefine $\Ffrak_{t,h}$ to be the filtration induced by \\
$\{z_1^i,r_1^i,\ldots,z_H^i,r_H^i\}_{i=1}^{t-1}$
and then repeat the arguments above verbatim.
\end{proof}

\subsubsection{Proof of Lemma \ref{lem:concentrate-2-routine}}
\label{sec:concentrate-2-routine}

\begin{proof}
Recall $\mu_f$ denotes the Nash policy of the max-player induced by $f$. If there exist more than one induced Nash policies, we can break the tie arbitrarily so that $\mu_f$ is uniquely defined for each $f\in\Fcal$. 
 Denote $\Pi_\Fcal:=\{\mu_f~|~f \in \Fcal\}$. We have $|\Pi_\Fcal|\le |\Fcal|$.

Consider a fixed tuple $(k,h,g,\mu)\in[K]\times[H]\times\Gcal\times\Pi_\Fcal$. Let 
\begin{align*}
	W_t(h,g,\mu) := \left(g_h(z_h^t)- r_h^t -Q^{\mu,\dagger}_{h+1}(s_{h+1}^t)\right)^2 
	- \left(\pf(Q^{\mu,\dagger}_h)(z_h^t)- r_h^t -Q^{\mu,\dagger}_{h+1}(s_{h+1}^t)\right)^2
	\end{align*}
and $\Ffrak_{t,h}$ be the filtration induced by 
$\{z_1^i,r_1^i,\ldots,z_H^i,r_H^i\}_{i=1}^{t-1}\bigcup\{z_1^t,r_1^t,\ldots,z_h^t\}.$
 We have 
\begin{align*}
 \E [W_t(h,g,\mu) \mid \Ffrak_{t,h}] =& \left[\left(g_h-Q^{\mu,\dagger}_{h}\right)(z_h^t)\right]^2
- \left[\left(\pf(Q^{\mu,\dagger}_{h}) -Q^{\mu,\dagger}_{h}\right)(z_h^t)\right]^2 \\
\ge &\left[\left(g_h-Q^{\mu,\dagger}_{h}\right)(z_h^t)\right]^2- \RealErr^2
\end{align*}
and by Cauchy-Schwarz inequality
\begin{align*}
&\Var [ W_t(h,g,\mu)\mid \Ffrak_{t,h}]\le \E [(W_t(h,g,\mu))^2 \mid \Ffrak_{t,h}] \\
\le & 36\left[\left(g_h-\pf(Q^{\mu,\dagger}_{h})\right)(z_h^t)\right]^2
\le 
\Ocal\left(\left[\left(g_h-Q^{\mu,\dagger}_{h}\right)(z_h^t)\right]^2+ \RealErr^2 \right). 
\end{align*} 
By Freedman's inequality, we have, with probability at least $1-\delta$, 
\begin{align*}
& -  \sum_{t=1}^{k} W_t(h,g,\mu) +\sum_{t=1}^{k}\left[\left(g_h-Q^{\mu,\dagger}_{h}\right)(z_h^t)\right]^2-k\RealErr^2 \\
 \le &
\Ocal \paren{\sqrt{\log(1/\delta)\sum_{t=1}^k\left(\left[\left(g_h-Q^{\mu,\dagger}_{h}\right)(z_h^t)\right]^2+ \RealErr^2 \right)} + \log(1/\delta)}.
\end{align*}

Now taking a union bound for all $(k,h,g,\mu)\in[K]\times[H]\times\Gcal\times\Pi_\Fcal$, we obtain that with probability at least $1-\delta$,  for  all $(k,h,g,\mu)\in[K]\times[H]\times\Gcal\times\Pi_\Fcal$
 \begin{equation}
 \label{eq:May17-1}
\begin{aligned}
& -  \sum_{t=1}^{k} W_t(h,g,\mu) +\sum_{t=1}^{k}\left[\left(g_h-Q^{\mu,\dagger}_{h}\right)(z_h^t)\right]^2-k\RealErr^2 \\
 \le & \Ocal \paren{\sqrt{\iota \sum_{t=1}^k\left(\left[\left(g_h-Q^{\mu,\dagger}_{h}\right)(z_h^t)\right]^2+ \RealErr^2 \right)} +\iota },
\end{aligned}
\end{equation}
where  $\iota=\log(HK|\Gcal||\Fcal|/\delta)$.
From now on, we will do all the analysis conditioning on this event being true. 

Since $\sum_{t=1}^k [(g_h-Q_{h}^{\mu,\dagger})(z_h^t)]^2$ is nonnegative, \eqref{eq:May17-1} implies for  all $(k,h,g,\mu)\in[K]\times[H]\times\Gcal\times\Pi_\Fcal$
$$
\sum_{t=1}^{k} W_t(h,g,\mu) \ge -\Ocal(\iota+k\RealErr^2 ).
$$
By choosing $\mu=\mu^k$, we have for  all $(k,h)\in[K]\times[H]$
$$
\sum_{t=1}^{k-1} \left(\pf(Q^{\mu^k,\dagger}_h)(z_h^t)- r_h^t -Q^{\mu^k,\dagger}_{h+1}(s_{h+1}^t)\right)^2
\le 
\inf_{g\in\Gcal_h} \sum_{t=1}^{k-1} \left(g(z_h^t)- r_h^t -Q^{\mu^k,\dagger}_{h+1}(s_{h+1}^t)\right)^2 + \Ocal(\iota+k\RealErr^2 ).
$$
We conclude the proof by recalling $\beta = \Theta\left(\log(HK|\Gcal||\Fcal|/\delta)+k\RealErr^2++k\CompErr^2\right)$.
\end{proof}

\section{Proofs for Section \ref{sec:example}}
\label{app:example}

In this section, we will first generalize  Bellman rank \citep{jiang2017contextual} to the setting of Markov Game. Then we show any problems of low Bellman rank also have low BE dimension. Finally, we prove all the examples in Section \ref{sec:example} have low Bellman rank, and thus low BE dimension. 

Denote $\Pi_\Fcal:=\{(\fmu_f,\fres_{f,g})~\mid~f,g\in\Fcal\}$, which is exactly the policy class that induces $\Dcal_\Fcal$.
\begin{definition}[Q-type $\epsilon$-effective Bellman rank]\label{def:effective-bellman-rank-typeI}
The $\epsilon$-effective Bellman rank is the minimum integer $d$ so that 
\begin{itemize}
\item 	There exists $\phi_h: \Pi_\Fcal\rightarrow\cH$ and $\psi_h:\Fcal\rightarrow\cH$ for each $h\in[H]$ where $\cH$ is a separable Hilbert space, such that for any $\pi\in\Pi$, $f,g\in\Fcal$, the average Bellman error
\begin{equation*}
\E_{\pi} [ (f_h-\Tcal_h^{\mu_g} f_{h+1})(s_h,a_h,b_h)]	=\langle \phi_h(\pi),\psi_h(f,g)\rangle_\cH~,
\end{equation*}
where $\|\psi_h(f,g)\|_\cH \le 1$.
\item $d = \max_{h\in[H]}d_{\rm eff}(\Xcal_h(\phi,\Fcal),\epsilon)$ where $\Xcal_h(\phi,\Fcal) = \{ \phi_h(\pi):\ \pi\in\Pi_\Fcal\}$. 
\end{itemize}
\end{definition}

\begin{proposition}[low Q-type effective Bellman rank $\subset$ low Q-type BE dimension]
\label{prop:effective-bellman-bedim}
If an MG with function class $\cF$ has Q-type  $\epsilon$-effective Bellman rank $d$, then
\begin{equation*}
\BEdim(\Fcal,\epsilon)\le d.
\end{equation*}
\end{proposition}
\begin{proof}
Assume there exists  $\pi_1,\ldots,\pi_n\in\Pi_{\Fcal}$ and  $(f^1,g^1)\ldots,(f^n,g^n)$ such that for all $t\in[n]$, $\sqrt{\sum_{i=1}^{t-1}(\E_{\pi_i}[f^t_h-\Tcal_h^{\mu_{g^t}} f_{h+1}^t])^2}\le \epsilon$ and $|\E_{\pi_t}[f^t_h-\Tcal_h^{\mu_{g^t}} f_{h+1}^t] |>\epsilon$.
	By the definition of effective Bellman rank, this is equivalent to: $\sqrt{\sum_{i=1}^{t-1}(\langle \phi_h(\pi^t),\psi_h(f^i,g^i)\rangle)^2}\le \epsilon$ and $| \langle \phi_h(\pi^t),\psi_h(f^t,g^t)\rangle|>\epsilon$ for all $t\in[n]$.
	
	For notational simplicity, define $x_i = \phi_h(\pi^i)$ and $\theta_i = \psi_h(f^i,g^i)$. Then
	\begin{equation}
		\left\{
		\begin{aligned}
		&\sum_{i=1}^{t-1} (x_i^\top 	\theta_t)^2 \le \epsilon^2, \quad t\in[n] \\
		& |x_t^\top 	\theta_t| \ge \epsilon, \quad t\in [n].
		\end{aligned}
\right.
	\end{equation}
Define $\Sigma_t  = \sum_{i=1}^{t-1} x_i x_i^\top + {\epsilon^2} \cdot \mathrm{I}$. We have 
\begin{equation}
		\begin{aligned}
		\|\theta_t\|_{\Sigma_t} \le \sqrt{2}\epsilon \quad \Longrightarrow \quad 
		 \epsilon\le |x_t^\top 	\theta_t| \le \|\theta_t\|_{\Sigma_t}\cdot   \|x_t\|_{\Sigma_t^{-1}}\le \sqrt{2}\epsilon \|x_t\|_{\Sigma_t^{-1}}, \quad t\in [n].
		\end{aligned}
	\end{equation}
As a result, we should have 	$\|x_t\|_{\Sigma_t^{-1}}^2\ge 1/2$ for all $t\in [n]$.
Now we can apply the standard log-determinant argument,
\begin{align*}
	\sum_{t=1}^{n} \log(1+ \|x_t\|_{\Sigma_t^{-1}}^2) = 
	\log\left(\frac{\det(\Sigma_{n+1})}{\det(\Sigma_1)}\right)
	=  \log\det\left(\mathrm I + \frac{1}{\epsilon^2} \sum_{i=1}^n x_i x_i^\top \right),
\end{align*}
which implies 
\begin{align}
	0.5 \le \min_{t\in[n]} \|x_t\|_{\Sigma_t^{-1}}^2 
	\le \exp\left(\frac{1}{n} \log\det\left(\mathrm I + \frac{1}{\epsilon^2} \sum_{i=1}^n x_i x_i^\top \right)\right)-1.
\end{align}
Choose $n=d_{\rm eff}(\Xcal_h(\phi,\Fcal),\epsilon)$ that is the minimum positive integer satisfying 
\begin{equation}
	\sup_{x_1,\ldots,x_n \in \Xcal_h(\phi,\Fcal)}\frac{1}{n} \log\det\left(\mathrm I + \frac{1}{\epsilon^2} \sum_{i=1}^n x_i x_i^\top \right) \le e^{-1}.
\end{equation}
This leads to a contradiction because $0.5> e^{e^{-1}}-1$. So we must have $n\le d_{\rm eff}(\Xcal_h(\psi,\Fcal),\epsilon).$
\end{proof}

Similarly, we can define V-type Bellman rank, and prove it is upper bounded by V-type BE dimension.

\begin{definition}[V-type $\epsilon$-effective Bellman rank]\label{def:effective-bellman-rank-typeII}
The $\epsilon$-effective Bellman rank is the minimum integer $d$ so that 
\begin{itemize}
\item 	There exists $\phi_h: \Pi_\Fcal\rightarrow\cH$ and $\psi_h:\Fcal\rightarrow\cH$ for each $h\in[H]$ where $\cH$ is a separable Hilbert space, such that for any $\pi\in\Pi$, $f,g,w\in\Fcal$, the average Bellman error
\begin{equation*}
 \E[(f_h-\Tcal_{h}^{\mu_g} f_{h+1})(s_h,a_h,b_h)~\mid~s_h\sim\pi, \ (a_h,b_h)\sim \mu_{g,h} \times\nu_{g,w,h}]	=\langle \phi_h(\pi),\psi_h(f,g,w)\rangle_\cH~,
\end{equation*}
where $\|\psi_h(f,g,w)\|_\cH \le 1$.
\item $d = \max_{h\in[H]}d_{\rm eff}(\Xcal_h(\phi,\Fcal),\epsilon)$ where $\Xcal_h(\phi,\Fcal) = \{ \phi_h(\pi):\ \pi\in\Pi_\Fcal\}$. 
\end{itemize}
\end{definition}

\begin{proposition}[low V-type effective Bellman rank $\subset$ low V-type BE dimension]
\label{prop:effective-bellman-bedim-V}
If an MG with function class $\cF$ has V-type $\epsilon$-effective Bellman rank $d$, then
\begin{equation*}
\dim_{\textrm{ VBE}}(\Fcal,\epsilon)\le d.
\end{equation*}
\end{proposition}
We omit the proof here because it is basically the same as Proposition \ref{prop:effective-bellman-bedim}.

\subsection{Proofs for Examples}

Below, we prove the problems introduced in Section \ref{sec:example} have either low Q-type or low V-type Bellman rank. Therefore, by Proposition \ref{prop:effective-bellman-bedim} and \ref{prop:effective-bellman-bedim-V}, they also have low Q-type or low V-type BE dimension. 

\begin{proof}[Proof of Proposition \ref{ex:tabular}]
We have 
\begin{equation}\label{eq:Jun2-1}
	\begin{aligned}
&		\E[ (f_h-\Tcal^{\mu_g}_h f_{h+1})(s_h,a_h,b_h)~\mid~(s_h,a_h,b_h)\sim\pi]\\
	= & \big\langle \Pr^\pi[(s_h,a_h,b_h)=\cdot]~,~(f_h-\Tcal^{\mu_g}_h f_{h+1})(\cdot) \big\rangle,
	\end{aligned}
\end{equation}
where the LHS only depends on $\pi$, and the RHS only depends on $f,g$.
\end{proof}

\begin{proof}[Proof of Proposition \ref{ex:linear}]
Consider two arbitrary $\theta_h,\theta_{h+1}\in B_d(R)$ and $g\in\Fcal$. By self-completeness, there exists $\theta_h^g\in B_d(R)$ so that 
$\Tcal_h^{\mu_g} (\phi_{h+1}\trans\theta_{h+1})=\phi_{h}\trans\theta_{h}^g$.  Therefore, we have
		\begin{equation}
		\begin{aligned}
			&		\E\big[ [\phi_{h}\trans\theta_{h}-\Tcal_h^{\mu_g} (\phi_{h+1}\trans\theta_{h+1})](s_h,a_h,b_h)~\mid~(s_h,a_h,b_h)\sim\pi]\\
	= & 		\E\big[ \phi_{h}(s_h,a_h,b_h)\trans(\theta_h-\theta_{h}^g)~\mid~(s_h,a_h,b_h)\sim\pi]\\
	= & \big\langle \E_\pi\big[ \phi_{h}(s_h,a_h,b_h)]~,~\theta_h-\theta_{h}^g \big\rangle,
		\end{aligned}
	\end{equation}
	where the LHS only depends on $\pi$, and the RHS only depends on $\theta_h,\theta_{h+1},g$.
\end{proof}

\begin{proof}[Proof of Proposition \ref{ex:kernel}]
Consider two arbitrary $\theta_h,\theta_{h+1}\in B_\cH(R)$ and $g\in\Fcal$. By self-completeness, there exists $\theta_h^g\in B_\cH(R)$ so that 
$\Tcal_h^{\mu_g} (\phi_{h+1}\trans\theta_{h+1})=\phi_{h}\trans\theta_{h}^g$.  Therefore, we have
		\begin{equation}
		\begin{aligned}
			&		\E\big[ [\phi_{h}\trans\theta_{h}-\Tcal_h^{\mu_g} (\phi_{h+1}\trans\theta_{h+1})](s_h,a_h,b_h)~\mid~(s_h,a_h,b_h)\sim\pi]\\
	= & 		\E\big[ \phi_{h}(s_h,a_h,b_h)\trans(\theta_h-\theta_{h}^g)~\mid~(s_h,a_h,b_h)\sim\pi]\\
	= & \big\langle \E_\pi\big[ \phi_{h}(s_h,a_h,b_h)]~,~\theta_h-\theta_{h}^g \big\rangle,
		\end{aligned}
	\end{equation}
	where the LHS only depends on $\pi$, and the RHS only depends on $\theta_h,\theta_{h+1},g$.
\end{proof}

\begin{proof}[Proof of Proposition \ref{ex:block}]
Let $d(i)$ denote the distribution of state $s$ given $q(s)=i$. 
For any policy $\pi$ and $f,g,w\in\Fcal$
	\begin{align*}
 & \E \left[ (f_h-\Tcal^{\mu_{g}} f_{h+1})(s_h,a_h,b_h)  ~\mid ~s_h \sim \pi,~ (a_h,b_h)\sim \mu_g\times\nu_{g,w}\right]	\\
 = & \big\langle \Pr^\pi_h[q(s_h)=\cdot]~,~ 
 \E \left[ (f_h-\Tcal^{\mu_{g}} f_{h+1})(s_h,a_h,b_h)  \mid s_h \sim d(\cdot),~ (a_h,b_h)\sim \mu_g\times\nu_{g,w}\right] \big\rangle,
\end{align*}
where the LHS only depends on $\pi$ while the RHS only depends on $f,g,w$. Both of them are $m$-dimensional.
\end{proof}

\begin{proof}[Proof of Proposition \ref{ex:selection}]
The case $h=1$ is trivial. We only need to consider $h\ge2$. 
For any policy $\pi$ and $f,g,w\in\Fcal$
	\begin{align*}
 & \E \left[ (f_h-\Tcal^{\mu_{g}} f_{h+1})(s_h,a_h,b_h)  ~\mid ~s_h \sim \pi,~ (a_h,b_h)\sim \mu_g\times\nu_{g,w}\right]	\\
= &\sum_{z\in\Scal\times\Acal\times\Bcal} \sum_{s} \Pr^\pi[(s_{h-1},a_{h-1},b_{h-1})=z] \times 
\Pr[s_{h}=s\mid (s_{h-1},a_{h-1},b_{h-1})=z]   \\
&\quad \quad \times\E \left[ (f_h-\Tcal^{\mu_{g}} f_{h+1})(s_h,a_h,b_h)  \mid s_h =s,~ (a_h,b_h)\sim \mu_g\times\nu_{g,w}\right]	\\
= &\sum_{z\in\Scal\times\Acal\times\Bcal} \sum_{s} \Pr^\pi[(s_{h-1},a_{h-1},b_{h-1})=z] \times 
 \langle \phi_{h-1}(z),\psi_{h-1}(s)\rangle  \\
&\quad \quad \times\E \left[ (f_h-\Tcal^{\mu_{g}} f_{h+1})(s_h,a_h,b_h)  \mid s_h =s,~ (a_h,b_h)\sim \mu_g\times\nu_{g,w}\right]	\\
= & \Big\langle \sum_{z\in\Scal\times\Acal\times\Bcal}  \Pr^\pi[(s_{h-1},a_{h-1},b_{h-1})=z] \times  \phi_{h-1}(z),\\
 & \quad\quad 
\sum_{s\in\Scal} \E \left[ (f_h-\Tcal^{\mu_{g}} f_{h+1})(s_h,a_h,b_h)  \mid s_h =s,~ (a_h,b_h)\sim \mu_g\times\nu_{g,w}\right]	\times \psi_{h-1}(s) \Big\rangle\\
= & \Big\langle   \E_\pi[\phi_{h-1}(s_{h-1},a_{h-1},b_{h-1})],\\
 & \quad\quad 
\sum_{s\in\Scal} \E \left[ (f_h-\Tcal^{\mu_{g}} f_{h+1})(s_h,a_h,b_h)  \mid s_h =s,~ (a_h,b_h)\sim \mu_g\times\nu_{g,w}\right]	\times \psi_{h-1}(s) \Big\rangle,
\end{align*}
where the LHS only depends on $\pi$ while the RHS only depends on $f,g,w$. 
By the regularization condition of Kernel MGs, the norm of the RHS is bounded by $2R+1$.
\end{proof}

\section{Proofs for Appendix~\ref{sec:type2}}
\label{sec:proof-type2}

In this section, we present the proof of Theorem~\ref{thm:golf-mg-regret-V} and Corollary \ref{cor:golf-mg-pac-V}. The techniques are basically the same as  those in Appendix~\ref{app:typeI-golf}.
Please notice that whenever coming across $\Dcal_{\Fcal}$ and $\Dcal_{\Delta}$ in this section, we use their definitions introduced in Appendix~\ref{sec:type2}.

To begin with, we have the following concentration lemma (akin to Lemma \ref{lem:concentrate-1-routine} and \ref{lem:concentrate-2-routine}) for the sub-routine \golfbr\ under the samples collected with Option II.

\begin{lemma}\label{lem:concentrate-routine-V}
 Under Assumption \ref{as:realizable} and \ref{as:complete}, if we choose $\beta=c\cdot(\log(KH|\Fcal||\Gcal|/\delta)+K\CompErr^2+K\RealErr^2)$ with some large absolute constant $c$ in Algorithm \ref{alg:golfmg} with Option II, then
		with probability at least $1-\delta$, for all $(k,h)\in[K]\times[H]$, we have 
		\begin{enumerate}[label=(\alph*)]
			\item $\sum_{i=1}^{k-1}  \E \left[ \paren{(\tilde{f}_h^k - \cT^{\mu^k}_h \tilde{f}_{h+1}^k)(s_h,a,b) }^2\mid s_h \sim \pi^i,(a,b) \sim \unif(\Acal \times \Bcal)\right]
			{\le} \mathcal{O} ( \beta)$,
			\item $\sum_{i=1}^{k-1} \sum_{(a,b)\in \Acal \times \Bcal} \paren{(\tilde{f}^k_h - \cT^{\mu^k}_h \tilde{f}_{h+1}^k)(s_h^i,a,b) }^2
			{\le} \mathcal{O} ( |\Acal||\Bcal|\beta)$,
			\item $\pf(Q^{\mu^k,\dagger})\in \cC^{\mu^k}$,
		\end{enumerate}
		where (i) $\pi^i=(\mu^i,\nu^i)$; (ii) $s_h^i$ denotes the state at step h, collected by following $\pi^i$ until step $h$, in the $i^{\rm th}$ outer iteration; (iii)  $\tilde f^k$ denotes the optimistic function computed in sub-routine \golfbr\ in the  $k^{\rm th}$ outer iteration.
	\end{lemma}
\begin{proof}
	To prove (a), we only need to redefine the filtration  $\Ffrak_{t,h}$ in Appendix \ref{sec:concentrate-1-routine} to be the filtration induced by  
$\{z_1^i,r_1^i,\ldots,z_H^i,r_H^i\}_{i=1}^{t-1}$ where $z_h^i=(s_h^i,a_h^i,b_h^i)$, and repeat the arguments therein verbatim. Similarly, for (b), we only need to redefine $\Ffrak_{t,h}$ in Appendix \ref{sec:concentrate-1-routine} to be the filtration induced by  
$\{z_1^i,r_1^i,\ldots,z_H^i,r_H^i\}_{i=1}^{t-1}\bigcup\{z_1^t,r_1^t,\ldots,z_{h-1}^t,r_{h-1}^t,s_{h}^t\}$. And the proof of (c) is the same as that of Lemma \ref{lem:concentrate-2-routine} in Appendix \ref{sec:concentrate-2-routine}.
\end{proof}

	\begin{proposition}[approximate best-response regret]
		\label{prop:routine-type2}
Under Assumption \ref{as:realizable} and \ref{as:complete}, 
there exists an absolute constant $c$ such that for any $\delta\in(0,1]$, $K\in\N$, 
if we choose $\beta=c\cdot(\log(KH|\Fcal||\Gcal|/\delta)+K\CompErr^2+K\RealErr^2)$, then
with probability at least $1-\delta$
$$
\sum_{k=1}^K \left(V^{\mu^k,\nu^k}_1(s_1) -  V^{\mu^k,\dagger}_1(s_1)\right)
\le \Ocal\left(  H\sqrt{|\Acal||\Bcal|k\beta\cdot \dim_{\textrm{VBE}}(\Fcal, 1/K)}\right).
$$
\end{proposition}
		
\begin{proof}
	By Lemma \ref{lem:concentrate-routine-V}, $\pf(Q^{\mu^k,\dagger})\in \Ccal^{\mu^k}$. Since \golfbr\ chooses $\tilde{f}^k$ optimistically, value difference lemma (Lemma~\ref{lem:value-difference}) gives  
		\begin{align*}
		& \sum_{k=1}^K \left(V^{\mu^k,\nu^k}_1(s_1) -  V^{\mu^k,\dagger}_1(s_1)\right)\\
		\le& \sum_{k=1}^K \left(V^{\mu^k,\nu^k}_1(s_1) - \min_{\nu}\D_{\mu^k_1 \times \nu} \tilde{f}_{1}^k(s_1)\right)+k\RealErr \\
		 =& -\sum_{k=1}^K\sum_{h=1}^H \E_{\pi^k}
		 \left[ (\tilde{f}^k_h -  \Tcal^{\mu^k}_h\tilde{f}_{h+1}^k)(s_h,a,b)~|~(a,b) \sim \mu_h^k \times \nu_{f^k,\tilde{f}^k,h}\right]+k\RealErr \\
		\le &- \sum_{h=1}^H \sum_{k=1}^K \E\left[ (\tilde{f}^k_h -  \Tcal^{\mu^k}_h\tilde{f}_{h+1}^k)(s_h^k,a,b)~|~(a,b) \sim \mu_h^k \times \nu_{f^k,\tilde{f}^k,h}\right]+\sqrt{KH\log(KH/\delta)}+k\RealErr ,
		\end{align*}
where the expectation in the second last line is over $s_h\sim\pi^k$ and $(a,b)\sim \mu_h^k \times \nu_{f^k,\tilde{f}^k,h}$, 
the expectation in the last line is only over $(a,b)\sim \mu_h^k \times \nu_{f^k,\tilde{f}^k,h}$ conditioning on $s_h$ being fixed as $s_h^k$, and the definition of $\nu_{f,g}$ is introduced in Section \ref{sec:main-results}.		
		
		By Jensen's inequality and Lemma \ref{lem:concentrate-routine-V} (b), we have for all $k\in[K]$
		\begin{align*}
			&\sum_{t=1}^k \left(\E\left[ (\tilde{f}^k_h -  \Tcal^{\mu^k}_h\tilde{f}_{h+1}^k)(s_h^t,a,b)~|~(a,b) \sim \mu_h^k \times \nu_{f^k,\tilde{f}^k,h}\right]\right)^2\\
			\le &\sum_{t=1}^k \E\left[ \left(\tilde{f}^k_h -  \Tcal^{\mu^k}_h\tilde{f}_{h+1}^k\right)^2(s_h^t,a,b)~|~(a,b) \sim \mu_h^k \times \nu_{f^k,\tilde{f}^k,h}\right]\\
			\le & \sum_{t=1}^k \sum_{(a,b)\in \Acal \times \Bcal} \left(\tilde{f}^k_h -  \Tcal^{\mu^k}_h\tilde{f}_{h+1}^k\right)^2(s_h^t,a,b)\le  \mathcal{O} \left( |\Acal||\Bcal|\beta\right).
		\end{align*}
		So we can apply Lemma \ref{lem:de-regret} with $\Gcal = \cH_{\cF,h}$ (here, $\cH_\Fcal$  refers to the one introduced  in Definition \ref{def:vbedim}), $\Pi = \Dcal_{\Delta,h}$, $\epsilon  = 1/K$ and obtain
		\begin{align*}
			&\sum_{k=1}^K \left|\E\left[ (\tilde{f}^k_h -  \Tcal^{\mu^k}_h\tilde{f}_{h+1}^k)(s_h^k,a,b)~|~ (a,b)\sim\mu_h^k \times \nu_{f^k,\tilde{f}^k,h}\right]\right|\\
			\le& \Ocal\left(\sqrt{|\Acal||\Bcal|\beta \cdot \dim_{\textrm{DE}}(\cH_{\cF,h},\Dcal_{\Delta,h},1/K) \cdot K}\right).
		\end{align*}
		Similarly by using Lemma \ref{lem:concentrate-routine-V} (a), we can show 
		\begin{align*}
			&\sum_{k=1}^K \E_{\pi^k}\left[ (\tilde{f}^k_h -  \Tcal^{\mu^k}_h\tilde{f}_{h+1}^k)(s_h,a,b)~|~(a,b) \sim \mu_h^k \times \nu_{f^k,\tilde{f}^k,h}\right]\\
			\le&\Ocal\left( \sqrt{|\Acal||\Bcal|\beta \cdot \dim_{\textrm{DE}}(\cH_{\cF,h},\Dcal_{\Fcal,h},1/K) \cdot K}\right).
		\end{align*}
		Putting all relations 		together as in Proposition \ref{prop:routine} and noticing that
		 $$k\RealErr \le  \Ocal\left(  H\sqrt{|\Acal||\Bcal|k\beta\cdot \dim_{\textrm{VBE}}(\Fcal, 1/K)}\right)$$
		completes the proof.
		\end{proof}


Equipped with the regret guarantee for  \golfbr, we are ready to bound the pseudo-regret for the main algorithm \golfmg. To begin with, we have the following concentration lemma (akin to Lemma \ref{lem:concentrate-1} and \ref{lem:concentrate-2}) under the samples collected with Option II.
\begin{lemma}\label{lem:concentrate-V}
Under Assumption \ref{as:realizable} and \ref{as:complete}, if we choose $\beta=c\cdot(\log(KH|\Fcal||\Gcal|/\delta)+K\CompErr^2+K\RealErr^2)$ with some large absolute constant $c$ in Algorithm \ref{alg:golfmg}, then
		with probability at least $1-\delta$, for all $(k,h)\in[K]\times[H]$, we have 
		\begin{enumerate}[label=(\alph*)]
			\item $\sum_{i=1}^{k-1}  \E [ \paren{(f_h^k- \cT_h f_{h+1}^k)(s_h,a,b) }^2\mid s_h \sim \pi^i,(a,b) \sim \unif(\Acal \times \Bcal)]
			{\le} \mathcal{O} ( \beta)$,
			\item $\sum_{i=1}^{k-1}\sum_{(a,b)\in \Acal \times \Bcal}  \paren{(f^k_h - \cT_h f_{h+1}^k)(s_h^i,a,b) }^2
			{\le} \mathcal{O} (|\Acal||\Bcal| \beta)$,
			\item $\pf(Q^\star)\in \Ccal^k$,
		\end{enumerate}
		where $s_h^i$ denotes the state at step h collected following $\pi^i$ until step $h$ in the $i^{\rm th}$ outer iteration. 
	\end{lemma}
\begin{proof}
	To prove (a), we only need to redefine the filtration  $\Ffrak_{t,h}$ in Appendix \ref{sec:concentrate-1} to be the filtration induced by  
$\{z_1^i,r_1^i,\ldots,z_H^i,r_H^i\}_{i=1}^{t-1}$ where $z_h^i=(s_h^i,a_h^i,b_h^i)$, and repeat the arguments therein verbatim. Similarly, for (b), we only need to redefine $\Ffrak_{t,h}$ in Appendix \ref{sec:concentrate-1} to be the filtration induced by  
$\{z_1^i,r_1^i,\ldots,z_H^i,r_H^i\}_{i=1}^{t-1}\bigcup\{z_1^t,r_1^t,\ldots,z_{h-1}^t,r_{h-1}^t,s_{h}^t\}$. And the proof of (c) is the same as that of Lemma \ref{lem:concentrate-2-routine} in Appendix \ref{sec:concentrate-2}.
\end{proof}

\begin{proof}[Proof of Theorem~\ref{thm:golf-mg-regret-V}]

The regret can be decomposed into two terms
\begin{align*}
\sum_{k=1}^K \left(V^{\star}_1(s_1) - V_1^{\mu^k,\dagger}(s_1)\right)
= \underset{\left( A \right)}{\underbrace{\sum_{k=1}^K \left(V^{\star}_1(s_1) - V_1^{\mu^k,\nu^k}(s_1) \right)}} + \underset{\left( B \right)}{\underbrace{\sum_{k=1}^K \left(V^{\mu^k,\nu^k}_1(s_1) -  V^{\mu^k,\dagger}_1(s_1)\right)}}.
\end{align*}

By Proposition \ref{prop:routine-type2}, $(B)$ can be upper bounded by with high probability,
	\begin{equation*}
		\sum_{k=1}^K \left(V^{\mu^k,\nu^k}_1(s_1) -  V^{\mu^k,\dagger}_1(s_1)\right) \le \Ocal \left(H\sqrt{|\Acal||\Bcal|K\beta \cdot \dim_{\textrm{VBE}}(\Fcal, 1/K)}\right).
	\end{equation*}

So it remains to control $(A)$. By minicking the proof of Theorem~\ref{thm:golf-online}, we have 
\begin{align*}
	&\sum_{k=1}^K [V^{\star}_1(s_1) - V_1^{\mu^k,\nu^k}(s_1)]\\
\le &	\sum_{k=1} ^K [V_{f^k,1}(s_1) - V_1^{\mu^k,\nu^k}(s_1)] +K\RealErr \\
= & \sum_{k=1} ^K\sum_{h=1}^H
\E_{\pi^k}\left[V_{f^k,h}(s_h) - 
r_h(s_h,a_h,b_h) - V_{f^k,h+1}(s_{h+1})\right]+K\RealErr\\
= & \sum_{k=1} ^K\sum_{h=1}^H
\E_{\pi^k}\left[\min_{\nu}\D_{\mu^k_h \times \nu} f_h^k(s_h) - 
\Tcal_h f_{h+1}^k(s_h,a_h,b_h) \right]+K\RealErr\\
\le & \sum_{k=1} ^K\sum_{h=1}^H
\E_{\pi^k}\left[f_h^k(s_h,a_h,b_h) - 
\Tcal_h f_{h+1}^k(s_h,a_h,b_h) \right]+K\RealErr\\
= & \sum_{k=1} ^K\sum_{h=1}^H
\E_{\pi^k}\left[(f_h^k-\Tcal_h f_{h+1}^k)(s_h,a,b) |(a,b) \sim \mu_h^k \times \nu_{f^k,\tilde f^k,h} \right]+K\RealErr\\
\le & \sum_{h=1}^H \sum_{k=1}^K \E\left[(f_h^k-\Tcal_h f_{h+1}^k)(s^k_h,a,b) |(a,b) \sim \mu_h^k \times \nu_{f^k,\tilde f^k,h} \right]+\Ocal\left(\sqrt{KH\log(KH/\delta)}\right)+K\RealErr,
\end{align*}
where the expectation in the second last line is over $s_h\sim\pi^k$ and $(a,b)\sim \mu_h^k \times \nu_{f^k,\tilde{f}^k,h}$, 
the expectation in the last line is only over $(a,b)\sim \mu_h^k \times \nu_{f^k,\tilde{f}^k,h}$ conditioning on $s_h$ being fixed as $s_h^k$, and the definition of $\nu_{f,g}$ is introduced in Section \ref{sec:main-results}.		

By Jensen's inequality and Lemma \ref{lem:concentrate-V} (b), we have for all $k\in[K]$
		\begin{align*}
			&\sum_{t=1}^k \left(\E\left[(f_h^k-\Tcal_h f_{h+1}^k)(s_h^t,a,b) ~|~(a,b) \sim \mu_h^k \times \nu_{f^k,\tilde f^k,h} \right]\right)^2\\
			\le &\sum_{t=1}^k \E\left[\left(f_h^k-\Tcal_h f_{h+1}^k\right)^2(s_h^t,a,b) ~|~(a,b) \sim \mu_h^k \times\nu_{f^k,\tilde f^k,h} \right]\\
			\le & \sum_{t=1}^k \sum_{(a,b)\in \Acal \times \Bcal} (f_h^k-\Tcal_h f_{h+1}^k)^2(s_h^t,a,b)
			\le  \mathcal{O} \left( |\Acal||\Bcal|\beta\right).
		\end{align*}

		So we can apply Lemma \ref{lem:de-regret} with $\Gcal = \cH_{\cF,h}$ (here, $\cH_\Fcal$ refers to the one introduced in Definition \ref{def:vbedim}), $\Pi = \Dcal_{\Delta,h}$, $\epsilon  = 1/K$ and obtain
		\begin{align*}
			&\sum_{k=1}^K \left|\E\left[(f_h^k-\Tcal_h f_{h+1}^k)(s^k_h,a,b) ~|~(a,b) \sim \mu_h^k \times \nu_{f^k,\tilde f^k,h} \right]\right|\\
			\le &\sqrt{|\Acal||\Bcal|\beta \cdot \dim_{\textrm{DE}}(\cH_{\cF,h},\Dcal_{\Delta,h},1/K) \cdot K}.
		\end{align*}

		Similarly by using Lemma \ref{lem:concentrate-V} (a), we can show 
		\begin{align*}
	&		\sum_{k=1}^K \left| \E_{\pi^k}\left[(f_h^k-\Tcal_h f_{h+1}^k)(s_h,a,b) ~|~(a,b) \sim \mu_h^k \times \nu_{f^k,\tilde f^k,h} \right]\right|\\
			\le &\sqrt{|\Acal||\Bcal|\beta \cdot \dim_{\textrm{DE}}(\cH_{\cF,h},\Dcal_{\Fcal,h},1/K) \cdot K}.
		\end{align*}
		
		Putting all relations together completes the proof.
\end{proof}

By an standard online-to-batch reduction, we can also prove the sample complexity guarantee.

\begin{proof}[Proof of Corollary~\ref{cor:golf-mg-pac-V}]
	We proceed as in the proof of Theorem~\ref{thm:golf-mg-regret-V} but take $\omega = \epsilon/H$ (instead of $\omega = 1/K$) every time we incur Lemma~\ref{lem:de-regret}.
	
	Theorem~\ref{thm:golf-mg-regret-V} essentially shows
$$
\sum_{k=1}^K \left(V^{\star}_1(s_1) - V_1^{\mu^k,\nu^k}(s_1)\right) \le \sum_{k=1}^K \left(V_{f^k,1}(s_1) - V_1^{\mu^k,\nu^k}(s_1)\right)
\le \Ocal(  H\sqrt{K\beta\cdot |\Acal||\Bcal|\dim_{\textrm{VBE}}(\Fcal, \epsilon/H)}+K\epsilon).
$$

Proposition~\ref{prop:routine-type2} essentially shows
\begin{align*}
	&\sum_{k=1}^K \left(V^{\mu^k,\nu^k}_1(s_1) -  V^{\mu^k,\dagger}_1(s_1)\right) \\
\le &\sum_{k=1}^K \left(V^{\mu^k,\nu^k}_1(s_1) -  \min_{\nu}\D_{\mu^k_{1}\times \nu} \tilde{f}_{1}^k(s_1)\right)\\
\le & \Ocal(  H\sqrt{K\beta\cdot |\Acal||\Bcal|\dim_{\textrm{VBE}}(\Fcal, \epsilon/H)}+K\epsilon).
\end{align*}
	
Comparing with the form in Theorem~\ref{thm:golf-mg-regret}, now we have an additional $\Ocal(K\epsilon)$ term, since we take $\omega = \epsilon/H$ when incurring Lemma~\ref{lem:de-regret}.
	
Putting them together and noticing by definition $\up{V}^k = V_{f^k,1}(s_1)$ and $\low{V}^k = \min_{\nu}\D_{\mu^k_{1}\times \nu} \tilde{f}_{1}^k(s_1)$, we can get 
$$
\frac{1}{K}\sum_{k=1}^K[\up{V}^k-\low{V}^k] \le \mathcal{O}(\frac{H}{K}\sqrt{|\Acal||\Bcal|\dim_{\textrm{VBE}}(\Fcal, \epsilon/H)K\beta}+\epsilon),
$$
with probability at least $1-\delta$, where $\beta=c\cdot(\log(KH|\Fcal||\Gcal|/\delta)+K\CompErr^2+K\RealErr^2)$.

By pigeonhole prinple, there must exist some $k$ s.t. 
$$\up{V}^k-\low{V}^k \le \Delta = c'(H\sqrt{|\Acal||\Bcal|\dim_{\textrm{VBE}}(\Fcal, \epsilon/H)\beta/K}+\epsilon).$$ Therefore, the output condition must be satisfied by some $k \in [K]$.

To make the right hand side order $\Ocal(\epsilon+ H\sqrt{d}(\RealErr+\CompErr) )$, it suffices to take 
$$
K \ge \Omega\big( (H^2d/\epsilon^2)\cdot\log(H|\Fcal||\Gcal|d/\epsilon)\big),
$$
where $d=\dim_{\textrm{VBE}}(\Fcal, \epsilon/H)$.
\end{proof}

\end{document}